\documentclass[journal]{IEEEtran}
\IEEEoverridecommandlockouts
\usepackage{amssymb}
\usepackage{amsmath}
\usepackage{amsthm}
\usepackage{graphicx}
\usepackage{algorithmic}
\usepackage{booktabs}
\usepackage{cite}
\usepackage{bm}
\usepackage{float}
\usepackage{multirow}
\usepackage{color}
\usepackage{subfigure}
\usepackage{caption}
\usepackage{epstopdf}
\usepackage{hyperref}
\usepackage{fixltx2e}
\usepackage{longtable}
\usepackage{diagbox}
\usepackage{changepage}
\usepackage{comment}
\usepackage{cases}
\usepackage[short]{newoptidef}
\usepackage{makecell}
\usepackage{mathabx}
\usepackage[linesnumbered,ruled]{algorithm2e}
\theoremstyle{definition}

\newtheorem{definition}{Definition}
\newtheorem{assumption}{Assumption}

\theoremstyle{revirtual simulatork}
\newtheorem{revirtual simulatork}{Revirtual simulatork}

\theoremstyle{plain}
\newtheorem{theorem}{Theorem}

\newtheorem{proposition}{Proposition}

\newcommand{\tmop}[1]{\ensuremath{\operatorname{#1}}}
\def\BibTeX{{\rm B\kern-.05em{\sc i\kern-.025em b}\kern-.08em
    T\kern-.1667em\lower.7ex\hbox{E}\kern-.125emX}}
\begin{document}

\title{\fontsize{23.3}{24}\selectfont Generative AI-empowered Simulation for Autonomous Driving in Vehicular Mixed Reality  Metaverses\\
\author{Minrui Xu, Dusit Niyato, \emph{Fellow, IEEE}, Junlong Chen, Hongliang Zhang, Jiawen Kang,\\ Zehui Xiong, Shiwen Mao, \emph{Fellow, IEEE}, and Zhu Han, \emph{Fellow, IEEE}
		\thanks{Minrui~Xu and Dusit~Niyato are with the School of Computer Science and Engineering, Nanyang Technological University, Singapore 308232, Singapore (e-mail: minrui001@e.ntu.edu.sg; dniyato@ntu.edu.sg).}
        \thanks{Hongliang~Zhang is with the School of Electronics, Peking University, Beijing 100871, China (e-mail: hongliang.zhang@pku.edu.cn).}
		\thanks{Junlong Chen and Jiawen~Kang are with the School of Automation, Guangdong University of Technology, China (e-mail: 3121001036@mail2.gdut.edu.cn; kavinkang@gdut.edu.cn).}
		\thanks{Zehui~Xiong is with the Pillar of Information Systems Technology and Design, Singapore University of Technology and Design, Singapore 487372, Singapore (e-mail: zehui\_xiong@sutd.edu.sg).}
		\thanks{Shiwen~Mao is with the Department of Electrical and Computer Engineering, Auburn University, Auburn, AL 36849-5201 USA (email: smao@ieee.org).}
		\thanks{Zhu~Han is with the Department of Electrical and Computer Engineering, University of Houston, Houston, TX 77004 USA, and also with the Department of Computer Science and Engineering, Kyung Hee University, Seoul 446-701, South Korea (e-mail: zhan2@uh.edu).}
	}
}

\maketitle

\begin{abstract}
In the vehicular mixed reality (MR) Metaverse, the distance between physical and virtual entities can be overcome by fusing the physical and virtual environments with multi-dimensional communications in autonomous driving systems. Assisted by digital twin (DT) technologies, connected autonomous vehicles (AVs), roadside units (RSU), and virtual simulators can maintain the vehicular MR Metaverse via digital simulations for sharing data and making driving decisions collaboratively. 
However, large-scale traffic and driving simulation via realistic data collection and fusion from the physical world for online prediction and offline training in autonomous driving systems are difficult and costly.
In this paper, we propose an autonomous driving architecture, where generative AI is leveraged to synthesize unlimited conditioned traffic and driving data in simulations for improving driving safety and traffic efficiency. First, we propose a multi-task DT offloading model for the reliable execution of heterogeneous DT tasks with different requirements at RSUs. Then, based on the preferences of AV's DTs and collected realistic data, virtual simulators can synthesize unlimited conditioned driving and traffic datasets to further improve robustness. 
Finally, we propose a multi-task enhanced auction-based mechanism to provide fine-grained incentives for RSUs in providing resources for autonomous driving. The property analysis and experimental results demonstrate that the proposed mechanism and architecture are strategy-proof and effective, respectively.

\end{abstract}


\section{Introduction}

The vehicular mixed reality (MR) Metaverse is envisioned as a promising solution for realizing autonomous driving by fusing the physical and virtual vehicular networks~\cite{zhou2022vetaverse, zhang2023location}. The multi-dimensional communications among physical and virtual entities can surrender the distance of ``data islands" on roads for
improving road safety and traffic efficiency while reducing energy consumption and carbon emissions~\cite{xu2022full}. Assisted by digital twin (DT) technologies, autonomous vehicles (AV) utilize advanced sensors, e.g., ultrasonic radars, cameras, and LiDAR, to collect data from their surrounding environments for constructing virtual representations in the virtual space~\cite{hui2022collaboration}. Then, AVs can make driving decisions, such as driving model selection and motion planning, via artificial intelligence (AI) methods. Even though panoramic cameras and high-class LiDAR are equipped with AVs, each AV can only collect limited environment data and cannot perceive the whole environment, e.g., occlusions~\cite{xiao2022perception}. Therefore, multiple connected AVs, roadside units (RSUs), and virtual simulators can share and fuse sensing data in the virtual space, to perceive the complete information of environments including occlusions. However, it is difficult and costly to collect realistic driving data on a large scale to train AVs directly in the physical world.

\begin{figure}[t]
    \centering
    \includegraphics[width=1\linewidth]{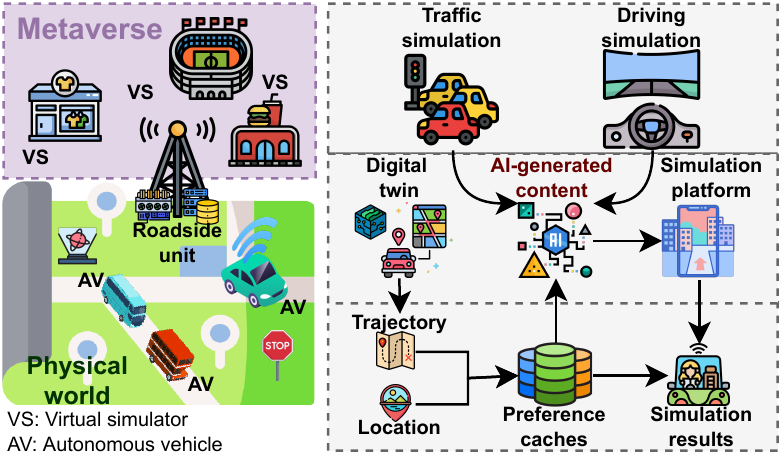}
    \caption{The MR Metaverse architecture of DT-assisted autonomous driving systems with traffic and driving simulations empowered by generative AI.}
    \label{fig:system}
\end{figure}

To address this issue, much effort from academia and industry has been devoted to developing platforms in the virtual space for traffic and driving simulations~\cite{feng2021intelligent, feng2018augmented}. By establishing virtual driving simulation platforms with DT~\cite{khan2022digital, hui2022collaboration} and MR~\cite{teng2021qoe, duan2022saliency} technologies, virtual representations of AVs can efficiently collect traffic and training data and cheaply test it on rare cases, such as virtual traffic accidents and car collisions under realistic scenes~\cite{niaz2021autonomous, drivesim}. 
Although traditional simulation platforms can generate an unlimited number of various driving experiences, the collected driving data requires a lot of manual work for labeling, which prevents the potential from being fully realized~\cite{feng2021intelligent}. Fortunately, with the multi-modal generative AI~\cite{gozalo2023chatgpt, xia2022generative, gilles2016computer}, the labeled traffic and driving data can be synthesized directly for virtual autonomous driving systems~\cite{kim2021drivegan}.
In this way, the process of using simulation platforms for autonomous driving training and evaluation is revolutionized by shifting from collecting and labeling data to directly synthesizing labeled data~\cite{kim2021drivegan, yang2020surfelgan}. Therefore, the simulation systems empowered by generative AI can generate large and diverse labeled driving datasets based on real-time road and weather conditions and user preferences for online prediction and offline training in autonomous driving systems.

Furthermore, in the vehicular Metaverse, connected AVs, RSUs, and virtual simulators need to construct the traffic and driving simulation platforms in the virtual space collaboratively. 
To update with virtual representations in virtual space, AVs continuously generate and offload multiple computation-intensive DT tasks to RSUs in online traffic simulation~\cite{hui2022collaboration}. Specifically, these DT tasks of each AV, including simulation, decision-making, and monitoring, are heterogeneous in requiring computing, communication resources, and deadlines.
In driving simulations, virtual simulators synthesize controllable traffic and driving data for satisfying specific requirements, e.g., passenger preferences and weather conditions, of the simulated driving tasks~\cite{drivesim}. In addition, the synthesized traffic and driving datasets can also be used in training virtual representations of AVs to further improve driving robustness.
These synchronization activities, e.g., DT task execution, traffic and driving simulations, and AV training, are demanding enormous communication and computing resources of RSUs for supporting autonomous driving systems~\cite{zhang2017non, zhang2022fruit}.
Therefore, developing effective multi-task incentive mechanisms that motivate RSUs to improve their use of communications and computing resources is imperative.

\begin{figure}[t]
    \centering
    \includegraphics[width=1\linewidth]{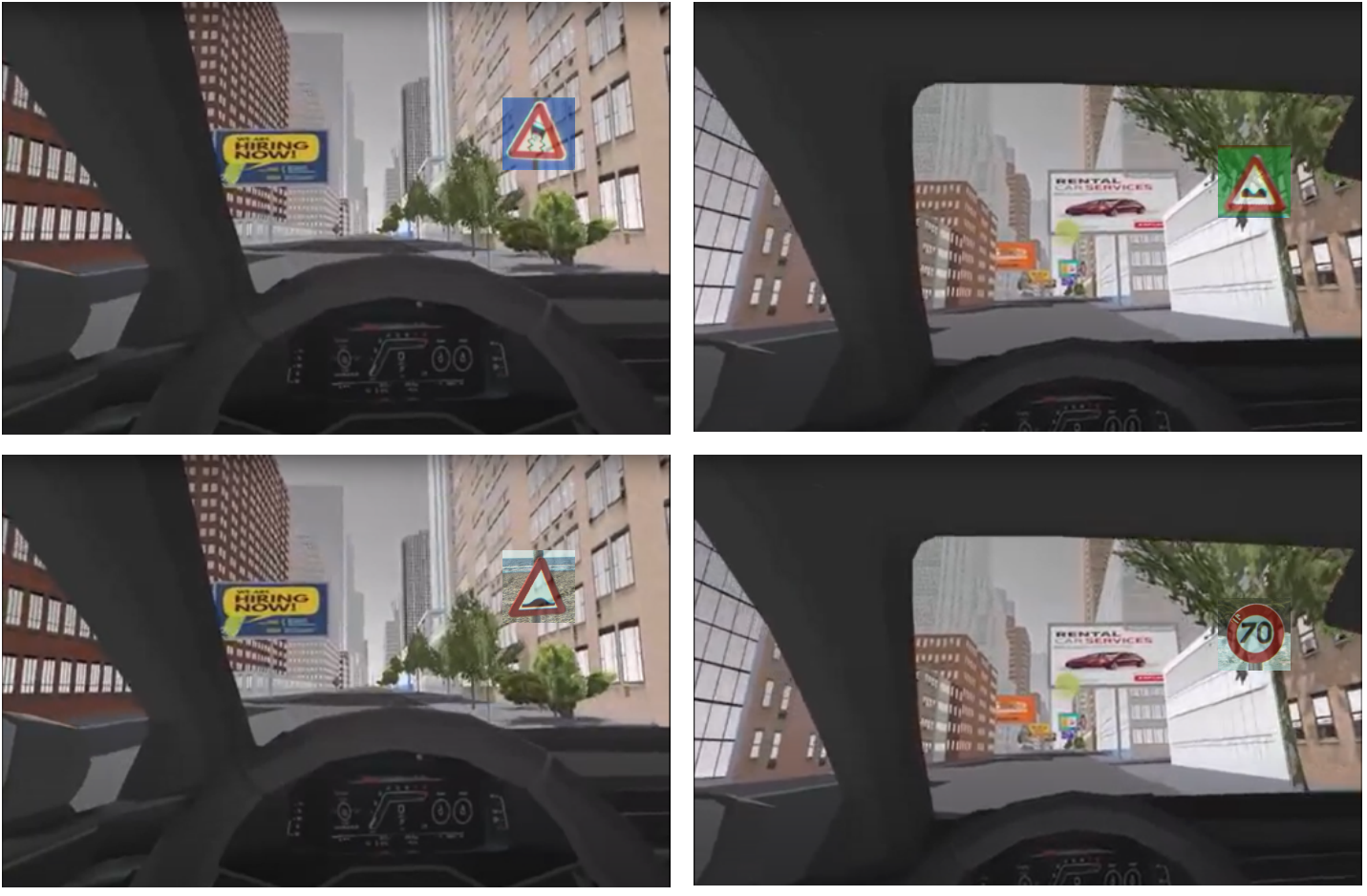}
    \caption{The screenshots of implemented driving simulation testbed~\cite{xu2022epvisa} with synthetic traffic signs generated by the proposed generative diffusion model, named TSDreamBooth.}
    \label{fig:simulation}
\end{figure}

As shown in Fig.~\ref{fig:system}, in this paper, we propose a novel DT-assisted autonomous driving architecture for the vehicular MR Metaverse, where generative AI is leveraged to synthesize massive and conditioned traffic and driving data for online and offline simulations. 
In detail, to improve reliability in DT task execution, we propose a multi-task DT offloading model where AVs can offload heterogeneous DT tasks with different deadlines to RSUs for real-time execution. 
To improve reliability in driving decision-making, virtual simulators can utilize the information in DTs, such as current location, historical trajectory, and user preferences, for online traffic simulations~\cite{zhong2022guided, yang2021hierarchical}. Moreover, based on the collected sensing data in the physical world and user preferences in DTs, virtual simulators can synthesize massive and conditioned driving data for AV training of virtual simulators via running generative AI models.
As a use case, we propose a diffusion model-based traffic sign generator, named TSDreamBooth, which is developed based on the DreamBooth~\cite{ruiz2022dreambooth} fine-tuned using Belgium traffic sign (BelgiumTS) dataset~\cite{mathias2013traffic}. The TSDreamBooth can be leveraged to generate virtual traffic sign images under different driving conditions and user preferences. Finally, we propose a multi-task enhanced auction-based mechanism to satisfy multi-dimensional requirements (e.g., prices and deadlines) of multiple DT tasks. We analyze the properties of the proposed auction and prove that it is strategy-proof and adverse-selection free. The experimental results demonstrate that the proposed framework can increase total social surplus by 150\%.


Our main contributions are summarized as follows:
\begin{itemize}
    \item To improve the safety and reliability of autonomous driving, we propose a novel DT-assisted MR Metaverse architecture with MR simulations empowered by generative AI. In this architecture, connected AVs, RSUs, and virtual simulators maintain digital simulation platforms in the virtual space, where data collecting, sharing, and utilizing among physical and virtual entities can improve driving safety and traffic efficiency in physical autonomous driving systems.
    \item In this architecture, we propose a reliable DT task offloading framework where AVs can continuously offload multiple DT tasks with different requirements to RSUs for updating DTs in the virtual space.
    \item In traffic and driving simulations, we consider generative AI-empowered virtual simulators to synthesize new driving data for AVs' decision-making and training. 
    \item To incentivize RSUs for providing resources in supporting autonomous driving systems, we propose a multi-task enhanced auction-based mechanism to offer fine-grained allocation results and prices for executing heterogeneous DT tasks with various deadlines. Based on the property analysis, the proposed mechanism is fully strategy-proof and adverse-selection free.
\end{itemize}
The rest of this paper is organized as follows. In Section~\ref{sec:related}, we review the related works. In Section~\ref{sec:system}, we discuss the proposed system architecture and its system model. Then, in Section~\ref{sec:mechanism}, we implement the multi-task enhanced auction-based mechanism. We demonstrate the experimental results in Section~\ref{sec:exp}, and provide a conclusion in Section~\ref{sec:con}.

\section{Related Works}\label{sec:related}

\subsection{DT-assisted Autonomous Driving}

In the vehicular MR Metaverse, DT technologies play an important role in assisting AVs to run more accurately and reliably in the physical transportation systems~\cite{jiang2022reliable}. In DT-assisted autonomous driving, driving data and AI techniques are both important to improve the capability and intelligence of AVs for real-time decision-making. Therefore, Niaz \textit{et al.} in~\cite{niaz2021autonomous} develop an autonomous driving test framework via DT technologies. Specifically, they consider using V2X communications to connect virtual space and physical space for driving safety and traffic efficiency improvement. Therefore, pure virtual driving, sensor data collecting, and real AV driving tests can be performed in this framework with limited test processes and working condition scenes. Furthermore, DT technologies can also be adopted in managing resources in vehicular networks. Specifically, Li \textit{et al.} in~\cite{li2022digital} propose a DT-driven computation offloading framework for minimizing computation latency and service discontinuity.
Considering social influence in vehicular networks, Zhang \textit{et al.} in~\cite{zhang2021digital} propose a DT-empowered content caching for improving caching scheduling in highly dynamic environments. In detail, the vehicular network is modeled as a digital twin where a learning-based caching algorithm is proposed to improve the system utility under dynamic content popularity, traffic density, and vehicle speed collaboratively.
However, existing work on digital twin-based autonomous driving systems only considers the single and homomorphic digital twin tasks in the system. This is incompatible with the highly heterogeneous computation in autonomous driving for adapting the dynamic vehicle status and driving environments. 






\subsection{Generative AI-empowered Autonomous Driving Simulation}

Through autonomous driving simulations, AVs can synthesize additional virtual driving experience for enhancing the inference and generalization capabilities of AI algorithms~\cite{chao2020survey}. The synthesized driving experiences and sensor data need to include not only rare conditions but also realistic observations that are similar to the scenes in the real world. Therefore, generative AI, such as generative adversarial networks (GANs) and diffusion models, is the promising solution to synthesizing new virtual driving data based on the existing realistic driving experience in the AVs' DTs. For instance, Kim \textit{et al.} propose the controllable simulation platform, named DriveGAN~\cite{kim2021drivegan}. Trained on 160 hours of real-world driving datasets, the DriveGAN can generate high-resolution and diverse simulations based on user-defined conditions, e.g., weather conditions and locations of simulation objects. Nevertheless, there is a gap between virtual and physical driving experiences and sensor data for AVs that are trained in simulation platforms~\cite{stocco2022mind}. 
Therefore, Yang \textit{et al.} in~\cite{yang2020surfelgan} propose a realistic sensor data synthesizing framework, named Surfelgan, for autonomous driving. The proposed data-driven data camera generation scheme demonstrates that the generated data can not only be visualized as high-quality data but also can be utilized as training datasets to improve the performance of AI algorithms in AVs. Furthermore, Zhong \textit{et al.}~\cite{zhong2022guided}  propose a generative diffusion model to synthesize controllable and realistic traffic simulations in autonomous driving systems. However, these simulation platforms cannot synthesize controllable and realistic driving experiences and sensor data. Therefore, they can merely synthesize and then label traffic and driving simulations for utilization in autonomous driving, rather than directly synthesizing labeled datasets based on specific requirements and conditions.


\subsection{Incentive Mechanisms in Connected Vehicular Networks}
With the goal of improving resource utilization in connected vehicular networks, incentive mechanisms are being developed to encourage RSUs to provide resources to vehicles~\cite{khan2022digital, khan2022digital2}. For example, Sun \textit{et al.} propose a preference-based incentive mechanism for resource allocation and scheduling in dynamic DT-assisted vehicular networks. In detail, the Stackelberg game is leveraged to formulate the interaction between leaders, i.e., vehicles, and followers, i.e., RSUs. Meanwhile, the Stackelberg game-based mechanisms can also be developed to encourage vehicles to participate in blockchain transactions and improve their utility~\cite{xing2022uavs}. In addition, storage resources in vehicular networks are valuable as the limited caching size and high deployment cost. In this regard, Xing~\textit{et al.} in~\cite{xing2022secure} propose a coalition formation game-based mechanism to motivate storage resource sharing in vehicular networks. Under this mechanism, vehicles can establish coalitions based on their routes for maximizing their utilities. In addition to storage resources, computing and communication resource sharing are considered in~\cite{hui2022collaboration}. Hui \textit{et al.} propose a collaboration-as-a-service framework in DT-enabled connected autonomous driving systems, where an auction game-based mechanism is proposed to obtain the Nash stable collaboration structure.  

However, previous works only focus on optimizing resource allocation in the physical world or building simulation platforms in the virtual world, while ignoring the potential of the synergistic effect between them. In this work, we propose a DT-assisted autonomous driving architecture empowered by generative diffusion models which can synthesize diverse and conditioned data for traffic and driving simulations. In addition, generated diffusion models are adopted to enhance the simulation capability for cross-modal traffic and driving data synthesizing.  Finally, we propose a multi-task enhanced second-score auction-based mechanism to provide fine-grained incentives for resources provided by RSUs.

\section{System Model}\label{sec:system}

In this section, we first give an overall description of the proposed system architecture consisting of connected AVs, RSUs, and virtual simulators for autonomous driving systems in Subsection~\ref{sec:architecture}. Then, we introduce the system model in this simulation system, including the network model in Subsection~\ref{sec:network}, the DT task model in Subsection~\ref{sec:dttask}, and the generative AI-empowered traffic and driving simulation models in Subsection~\ref{sec:simulation}. Finally, we formulate the incentive problems in Subsection~\ref{sec:problem} for RSUs with regards to the social surplus obtained by the provisioned resources to collaborate with autonomous driving systems.

\subsection{The Architecture of DT-assisted Autonomous Driving} \label{sec:architecture}

To enable autonomous driving in the vehicular MR Metaverse, connected AVs, RSUs, and virtual simulators can work together to maintain digital simulation platforms to share data and make AI-driven driving decisions. In this architecture, RSUs with sufficient communication and computing resources can provide online and offline simulation services for AVs and virtual simulators. Through the communication and computing resources of RSUs, AVs can maintain their digital representations in the virtual space. Specifically, AVs continuously generate digital representations during their traveling, which are offloaded to RSUs for remote executions within required deadlines~\cite{chen2020massive, wu2020massive}. During the execution of DTs, online simulations can be performed to improve the performance of decision-making modules in making driving decisions. As illustrated in Fig.~\ref{fig:full}, virtual simulators can use the available resources and time of RSUs to improve training modules of AVs via offline simulations. During the offline simulations, virtual driving environments are adopted to let AVs collect training data. Empowered by generative AI models, massive and conditioned traffic and driving can be synthesized and collected in the datasets for AVs' decision-making training modules. Therefore, more high-quality and diverse driving experiences can be leveraged by virtual simulators to train AVs. Finally, the simulation results are sent back to AVs for future utilization. In this architecture, online decision-making and offline training via traffic and driving simulations can improve driving safety and traffic efficiency in autonomous driving systems.
\begin{figure}
    \centering
    \includegraphics[width=1\linewidth]{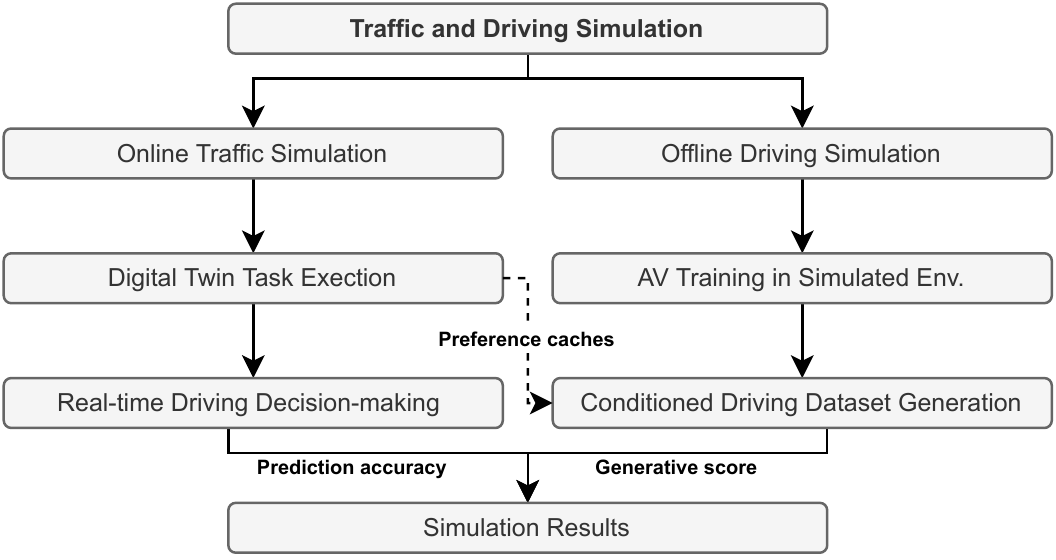}
    \caption{The workflow of DT-assisted autonomous driving simulation platforms empowered by generative AI.}
    \label{fig:full}
\end{figure}

In the system model, we consider three main roles in the vehicular MR Metaverse, i.e., AVs, RSUs, and virtual simulators. The set of $I$ AVs is represented by the set $\mathcal{I}=\{1,\dots, i,\dots, I\}$, the set of $J$ RSUs is represented as $\mathcal{J} = \{1,\dots,j,\dots,J\}$, and the set of $K$ virtual simulators is represented as $\mathcal{K}=\{0, 1,\dots, k, \dots, K\}$. We consider the RSUs to own adequate communication and computing resources for enabling autonomous driving systems, i.e., the resources of RSUs are enough for executing all the computation tasks of AVs within deadlines. To facilitate autonomous driving systems, both uplink and downlink communication channels are allocated to upload DT tasks and stream simulation results. Therefore, communication resources at RSU $j$ consist of uplink bandwidth $B_j^u$ and downlink bandwidth $B_j^d$. Moreover, to provide services such as executing DT tasks and simulating virtual traffic and driving, each RSU $j$ is equipped with computing resources, including the CPU frequency $f^{C}_j$ and the GPU frequency $f^{G}_j$. 

In autonomous driving, AVs maintain the DTs in the virtual space and continuously update the DTs by executing DT tasks, e.g., simulation, decision-making, and monitoring. These DT tasks require heterogeneous resources with various deadlines. Therefore, in the system model, $N_i$ DT tasks can be generated by AV $i$, which can be represented as $DT_i = (\textless s_{i,1}^\emph{DT}, e_{i,1}^\emph{DT}, d_{i,1} \textgreater, \ldots, \textless s_{i,n}^\emph{DT}, e_{i,n}^\emph{DT}, d_{i,n} \textgreater,$ $\ldots, \textless s_{i, N+i}^\emph{DT}, e_{i, N_i}^\emph{DT}, d_{i, N_i} \textgreater)$, where $s_{i, n}^\emph{DT}$ is the size of DT data, $e_{i, n}^\emph{DT}$ represents the number of CPU cycles required per unit data, and $d_{i, n}$ denotes the deadline for completing the task. As part of the DT data from AVs, there are preference caches that store passenger preferences, interests, and behaviors. This information is used to personalize the user's experience with DT and provide them with relevant and targeted content, services, and advertising. The size of preference caches of AV $i$ within the $DT_i$ is $C_i$. Each AV $i\in\mathcal{I}$ has its private value $v_i$ for executing its DT task $DT_i$, drawn from the probability distributions. The values of DT tasks can be interpreted as the characteristics of the AVs, such as the level of urgency to align with DT models~\cite{hui2022collaboration}, which may vary for each AV during its travel. 

We consider two types of virtual simulators in the vehicular MR Metaverse, i.e., driving virtual simulators and traffic virtual simulators. Traffic virtual simulators $1,\dots, K$ provide online traffic simulation designed to assist real-time decision-making of AVs, including driving mode selection, information fusion, and motion planning~\cite{xiao2022perception}. Online decision-making module via traffic simulation uses data collected from AVs, RSUs, and virtual simulators to create a simulated driving environment for testing and validating driving decisions and improving the safety and reliability of autonomous driving systems. As a result, the feedback of these driving decisions is immediately returned and perceived by the AVs. The driving virtual simulator $0$ delivers offline driving simulation to provide training simulation platforms to AVs. In the offline training module, the AI algorithms of AVs are trained by simulated driving practice. However, the performance improvement of AVs in the future still needs to be tested and validated, and thus the AVs cannot perceive immediate returns. The value of simulations for each simulation pair of AV $i$ and virtual simulator $k$ is $U_{i,k}$, which is the product of the common value $v_{i}$ of AV $i$ and the match quality $m_{i,k}$, i.e., $U_{i,k} = v_{i}m_{i,k}$. The common values for every virtual simulator $k$ are gained from the provisioning of traffic simulation for the  AV $i$, which can be represented by the AV $i$'s private value $v_i$~\cite{arnosti2016adverse}. Additionally, the amount of personalized information determines the match quality $m_{i,k}$ of virtual simulator $k$. This way, the values of AVs and virtual simulators in autonomous driving systems are positively correlated. Finally, let $U_{\iota,(l)}$ and $m_{\iota,(l)}$ represent the $l$ highest value and match quality for the  AV $\iota$, respectively.

\begin{figure}
    \centering
    \includegraphics[width=1\linewidth]{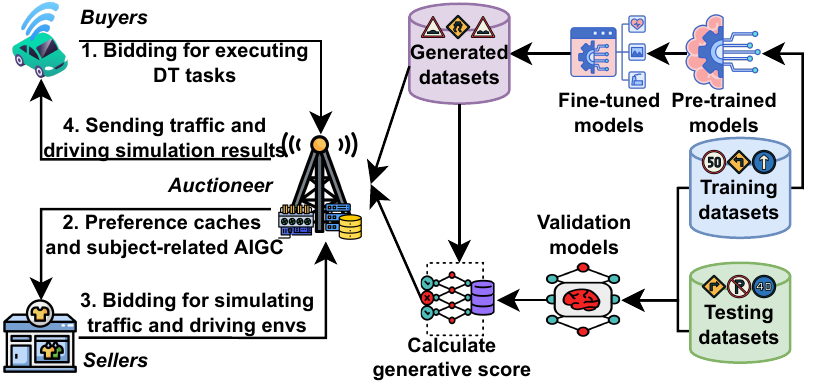}
    \caption{The workflow of the proposed  market}
    \label{fig:workflow}
\end{figure}

\subsection{Network Model}\label{sec:network}

In autonomous driving systems, cooperative vehicular networks are utilized for updating DTs and streaming simulation results~\cite{zhang2021energy, gao2022energy}, respectively. The channel gain between AV $i$ and RSU $j$ is represented by $g_{i,j}$, and the downlink transmission rate can be calculated as $R^{d}_{i,j} = B_j^d \log(1+\frac{g{i,j}P_j}{\sigma_i^2})$, where $\sigma^2_i$ is the additive white Gaussian noise at AV $i$. Additionally, the transmit power of AV $i$ is represented by $p_i$, and the uplink transmission rate can be calculated as $R^{u}_{i,j} = B_j^u \log(1+\frac{g{i,j}p_i}{\sigma_j^2})$, where $\sigma^2_j$ is the additive white Gaussian noise at RSU $j$.

\subsection{Multi-task Digital Twin Model}\label{sec:dttask}

In the multi-task digital twin model, the AVs update the DTs in virtual space by performing multiple DT tasks for different driving functions with heterogeneous levels of complexity and level of urgency.

\subsubsection{DT Task Execution}
To maintain the digital representation with the vehicular MR Metaverse, physical entities, i.e., AVs, generate and offload DT tasks to RSUs for remote execution. Therefore, we consider the demands as tasks that are required to be accomplished by RSUs.  The transmission latency $t_{i,n,j}^\emph{DT}$ for AV $i$ to upload its DT task $\textless s_{i,n}^\emph{DT}, e_{i,n}^\emph{DT}, d_{i,n} \textgreater$ to RSU $j$ can be calculated as~\cite{hui2022collaboration} 
$
      t_{i,n,j}^\emph{DT} = \frac{s_{i, n}^\emph{DT}}{R_{i,j}^{u}},
$
  where $R_{i,j}^{u}$ is the downlink transmission rate between AV $i$ and RSU $j$. 
  After completing the upload of the DT task, RSU $j$ uses its computing resources $f^{C}_{j}$ to execute the received DT tasks. The computation latency in processing the DT task $DT_i$ of AV $i$ for RSU $j$ can be calculated as 
$
      l_{i,n,j}^\emph{DT} = \frac{s_{i, n}^\emph{DT}e_{i, n}^\emph{DT}}{f^{C}_{j}}.
$
In the proposed system, without loss of generality, we consider that each RSU can allocate the virtual machines that have the capability to accomplish DT tasks within its required deadline~\cite{hui2022collaboration}, i.e., $T^{DT}_{i,n,j} =  t_{i,n,j}^\emph{DT} + l_{i,n,j}^\emph{DT} \leq d_{i, n}, \forall i\in \mathcal{I}, j\in\mathcal{J}, n=1, \ldots, N$. In addition, virtual simulators can provide traffic and driving simulation services to AVs with the remaining available communication and computing resources of RSUs. After traffic and driving simulation, virtual simulators send the simulation results to AVs for further utilization.
  

\subsection{Traffic and Driving Simulation Model}\label{sec:simulation}
  
\subsubsection{Generative AI-empowered Simulation}
As depicted in Fig. \ref{fig:workflow}, the generative AI-based traffic and driving simulations comprise training, fine-tuning, and inference stages. In the first step, the low-resolution text-image model is fine-tuned using input images paired with a text prompt containing a unique identifier and the class name of the subject~\cite{ruiz2022dreambooth}. A class-specific prior preservation loss is built in to take advantage of the semantic prior the model has over the class and make it create different instances belonging to the subject's class. In the second step, the super-resolution components of the model are fine-tuned using low and high-resolution image pairs from the input images. In this way, the model can maintain high accuracy on small details of the subject while creating different instances of the subject in different scenarios. Therefore, this process of fine-tuning a text-image diffusion model using data from a virtual simulator enables the creation of more accurate and diverse driving simulations, helping to improve the development and testing of AVs. In detail, virtual simulators adopt the Prior-Preservation Loss proposed in~\cite{ruiz2022dreambooth} to fine-tune the pre-trained models for customization of the local traffic signs.
        
\subsubsection{TSDreambooth}
During the fine-tuning process of generative AI, virtual simulators use their original simulation datasets as training inputs for the generative models. By utilizing the knowledge of the driving simulation, such as the class of traffic signs, the fine-tuned generative AI model for vehicular networks can effectively extract the features of these traffic signs. In this context, we propose the TSDreamBooth, fine-tuned on the traffic sign datasets. In the driving simulation, virtual simulators can use TSDreamBooth to generate a large amount of synthetic driving data based on local traffic signs using user preferences in AVs as input. In detail, RSUs extract preferences from DTs of AVs, known as preference caches. The preferences of AVs are collected by leveraging some user analysis equipment, such as eye-tracking devices. These preferences are then input into generative AI models as text prompts to produce diverse and conditioned simulation results. This enables virtual simulators to generate unlimited AV training experiments based on AV requirements similar to that are collected from realistic environments. As a result, the number of driving experiences for offline training is no longer limited to the hit preference caches \cite{xu2022epvisa}. However, due to the limitations of generative AI models, some simulated driving experiences may not meet expectations and can be identified by trained validation models.

The models generated are based on probability distributions, and thus the results produced by TSDreamBooth are not deterministic. The results of TSDreamBooth may not be the same every time that the model is run. Virtual simulators have the potential to capture variability in results and provide a better understanding of the uncertainties in the model's predictions by interrogating multiple results. Therefore, the validation models indicate the quality of generative AI models with generative score $G_{i,j,k}\in [0, 1]$, as demonstrated in Fig.~\ref{fig:workflow}. For each simulation result of virtual simulator $k$, the simulation task can be represented by $SIM_k = \textless s_k^\emph{SIM}, e_k^\emph{SIM} \textgreater$~\cite{ren2020edge}, where $s_k^\emph{SIM}$ is the data size of each simulation and $e_k^\emph{SIM}$ is the required GPU cycles per unit data for offline simulation. Therefore, given the total number of virtual simulators $K+1$, the match quality $m_{i,k}$ and hit preference caches $h_{i,k}$ are drawn independently from a set of distributions $m_{i,k} = h_{i,k} \sim F_{i,k}$. To explain further, given the  AV $\iota$, the traffic virtual simulators $k = 1,\dots, K$ can measure the match qualities $m_{\iota, k}$ of their traffic simulation. However, the driving virtual simulator $0$ that provides driving simulation to the  AV $\iota$ cannot immediately measure its match quality $m_{\iota, 0}$. Therefore, asymmetric information exists among virtual simulators that might result in adverse selection~\cite{arnosti2016adverse}.

Empowered by generative AI models, the match quality $m_{i,k}$ is no longer limited by the hit preference caches $h_{i,k}$. As generative AI can generate countless and diverse simulation results based on user preferences and location datasets, virtual simulators can utilize more computing resources and downlink transmission resources during offline training. During the remanding time of DT execution, the total number of simulations $Q_{i,n,j,k}$ can be calculated $Q_{i,n,j,k} = (d_{i,n}-T_{i,n,j}^\emph{DT})R_{i,j}^\emph{SIM}/s_k^\emph{SIM}$ for task $n$ in $DT_i$ of AV $i$ and its RSU $j$. Then, the marginal generative AI-empowered match quality of AV $i$ in simulator $k$ via RSU $j$ can be measured as
\begin{equation}
    m_{i,n,j, k} =\frac{\log_2 (1+G_{i,j,k} Q_{i,n,j,k}) h_{i,k}}{\theta(h_{i,k})},
\end{equation}
where $\theta(h_{i,k})$ is relative accuracy among the original model $w_{i}$ and the fine-tuned model $w_{i,k}$ for strongly convex objective~\cite{tran2019federated, zhang2021optimizing}. In detail, $\theta(\cdot) = 1$ indicates no improvement for training in simulation platforms, and $\theta(\cdot) = 0$ indicates the AI model is trained to be trained optimally.

\subsubsection{Simulation and Offline Training Model}

The effective transmission latency in simulation and transmitting the driving simulation $SIM_k$ to AV $i$ for task $n$ from RSU $j$ can be calculated as 
		\begin{equation}
			t_{i,j,k}^\emph{SIM} = \frac{Q_{i,n,j,k}s_k^\emph{SIM}}{R_{i,j}^{d}},\label{eq:arcomm}
		\end{equation}
    where $R_{i,j}^{d}$ is the downlink transmission rate between AV $i$ and RSU $j$. 
Moreover, the effective computation latency in simulation the driving simulation $SIM_k$ can be calculated as
		\begin{equation}
			l_{i,j,k}^\emph{SIM} = \frac{Q_{i,n,j,k}s_k^\emph{SIM}e_k^\emph{SIM}}{f^G_{j}},\label{eq:arcomp}
		\end{equation}
		which depends on the simulation latency in GPUs of RSU $j$. Eqs.~(\ref{eq:arcomm}) and~(\ref{eq:arcomp}) imply that the offline simulations in generative AI-empowered vehicular MR Metaverse can improve the utilization of communication and computing resources in autonomous driving systems.
		
In autonomous driving systems, RSUs can use their available computation and communication resources to provide real-time physical-virtual  services for AVs and virtual simulators. However, the total  latency cannot exceed the required deadline of AV $i$. Let $g_{i,j}^\emph{DT}$ be the allocation variable that AV $i$ is allocated to RSU $j$ and $g_{i,j,k}^\emph{SIM}$ be the allocation variable that virtual simulator $k$ is allocated by RSU $j$ to match AV $i$. The total latency $T^\emph{total}_{i,j,k}$ required by RSU $j$ to process both the DT task of AV $i$ and the simulations of virtual simulator $k$ should be less than the required deadline, which can be expressed as
\begin{equation}
  \begin{aligned}
      T^\emph{total}_{i,n,j,k} = g_{i,j}^\emph{DT}&\cdot(t_{i,n,j}^\emph{DT} + l_{i,n,j}^\emph{DT}) \\&+ g_{i,j,k}^\emph{SIM}\cdot(t_{i,n,j,k}^\emph{SIM} +l_{i,n,j,k}^\emph{SIM}) \leq d_{i, n},
  \end{aligned}
		\end{equation}
$ \forall i \in \mathcal{I}, j\in\mathcal{J}, k\in\mathcal{K}, n=1,\ldots,N$. 
The driving simulation of virtual simulator $k$ is running in the background of AV $i$ during the processing of DT tasks at RSU $j$, and thus the expected duration of offline training can also be represented by $T^\emph{total}_{i,n,j,k}$. 
  
\subsection{Problem Formulation}\label{sec:problem}


In the proposed system, a resource market, consisting of the online and offline submarkets, is established to incentivize RSUs to provide communication and computing resources for traffic and driving simulation for AVs and virtual simulators. Here, we consider participants in the market to be risk neutral, and their surpluses are correlated positively. Therefore, the mechanism is expected to map the DT values $\mathbf{v} = (v_1, \ldots, v_I)$ and simulation values $\mathbf{U}=(I_{1,0}, \ldots, U_{I,K})$ to the payments of AVs $\mathbf{p}^\emph{DT}=(p_1^\emph{DT}, \ldots, p_I^\emph{DT})$ and the payments of virtual simulators $\mathbf{p}^\emph{SIM}=(p_1^\emph{SIM}, \ldots, p_K^\emph{SIM})$ with the allocation probabilities $\mathbf{g}^\emph{DT} = (g^\emph{DT}_1,\dots, g^\emph{DT}_I)$ and $\mathbf{g}^\emph{SIM} = (g^\emph{SIM}_0, \dots, g^\emph{SIM}_K)$. By accomplishing DT tasks, the total expected surplus for RSUs from AV $i\in \mathcal{I}$ in the online submarket can be represented by $S^\emph{DT}(\mathbf{g}^\emph{DT}) = \mathbb{E}\left[\sum_{i=1}^{I} \mathcal{R}_{i,j} v_i \mathbf{g}^\emph{DT}_{i,j}(\mathbf{v})\right]$. Based on the optimal reaction to the dominant strategies of the traffic virtual simulators, the driving virtual simulator can motivate RSU with the expected surplus of $S_\emph{D}^\emph{SIM} = \mathbb{E}[U_{i,0}g^\emph{SIM}_{i,n,j, 0}(Q_i)]$. In addition, the total expected surplus provided by traffic virtual simulators is defined by  $S_\emph{T}^\emph{SIM}(\mathbf{g}^\emph{SIM}) = \mathbb{E}[\sum_{k=1}^{K}U_{i,k}g^\emph{SIM}_{i,n,j,k}(U_i)]$. In addition, we consider a cost-per-time payment model for simulation platforms, i.e., users pay a fee for each unit of time they use the simulation platform, e.g., per minute or per hour. Since AVs can only access the platforms provided by virtual simulators for a limited time $T$ while driving, the cost-per-time payment model is a viable and flexible solution for renting virtual simulation platforms. In conclusion, the social surplus that RSU $j$ can gain from the offline submarket can be defined as $S^\emph{SIM}(\mathbf{g}^\emph{SIM}) = T\cdot(\gamma S_\emph{D}^\emph{SIM}(\mathbf{g}^\emph{SIM}) + S_\emph{P}^\emph{SIM}(\mathbf{g}^\emph{SIM}))$, where $\gamma$ denotes the relative bargaining power of driving virtual simulator $0$.

To maximize the social surplus in the  market, the non-cooperative game among AVs, virtual simulators, and RSUs in the  mechanism $\mathcal{M}=(\mathbf{g}^{DT}, \mathbf{g}^{SIM}, \mathbf{p}^{DT}, \mathbf{p}^{SIM})$ can be formulated as
\begin{maxi!}|s|[2]<b>
    {\mathcal{M}}{S^\emph{DT}+ \sum_{n=1}^{N}T^\emph{total}_{i,n,j,k} \cdot \big(\gamma S_\emph{D}^\emph{SIM}+S_\emph{T}^\emph{SIM}\big)\label{eq:obj}}{}{}
    \addConstraint{T^\emph{total}_{i,n,j,k}}{\leq d_{i,n}\label{eq:con1}}
    \addConstraint{h_{i,k}}{\leq C_i\label{eq:con2}}
    \addConstraint{0\leq b^\emph{DT}_{i}}{\leq v^\emph{DT}_{i}\label{eq:con3}}
    \addConstraint{0\leq p^\emph{SIM}_{k}}{\leq U^\emph{SIM}_{\iota,k}\label{eq:con4}}
    \addConstraint{\sum_{i=1}^{I}g_{i,j}^\emph{DT}}{\leq 1\label{eq:con5}}
    \addConstraint{\sum_{k=0}^{K}g_{i,j,k}^\emph{SIM}}{\leq 1\label{eq:con6}}
    \addConstraint{g_{i,j}^\emph{DT}, g_{i,j,k}^\emph{SIM}}{\in \{0,1\},\label{eq:con7}}
\end{maxi!}
$\forall i\in \mathcal{I}, j\in\mathcal{J}, k\in\mathcal{K}, n=1,\ldots,N$. Constraint (\ref{eq:con1}) ensures the reliability of each DT task that can be accomplished within the required deadline. Constraint (\ref{eq:con2}) guarantees that the number of hit preference caches is less than the size of preference caches. Pricing constraints~(\ref{eq:con3}) and (\ref{eq:con4}) are listed to guarantee the individual rationality (IR) of traders. Allocation constraints (\ref{eq:con5}), (\ref{eq:con6}), and (\ref{eq:con7}) guarantee that each physical or virtual entity can be assigned by one and only one RSU.

In the simulation market, there exist two issues, i.e., externalities and asymmetric information, causing adverse selection in the social surplus maximization problem formulated in Eq.~(\ref{eq:obj}). First, adverse selection, as described in~\cite{akerlof1978market} refers to a market situation where participants with asymmetric information are only willing to pay the average market price. This can lead to an inefficient allocation and matching outcome in the market. In the context of traffic and driving simulations for autonomous driving, the physical and virtual entities (AVs and virtual simulators) have positively correlated surpluses for services. This means that the surplus of AVs in the online submarket can impact the surplus of virtual simulators in the offline submarket by affecting the common valuation of driving simulations. This correlation introduces externalities and asymmetric information for allocating physical and virtual entities in the service market.
\begin{itemize}
    \item Externalities: The externalities are introduced to the online submarket from the offline submarket. The traffic and driving simulation results of virtual simulators have different match qualities for different physical AVs. However, during the allocation of AV in the physical market, the  virtual simulator is unknown for participants in the online submarket, which might affect the total processing latency in the online submarket. Therefore, AVs in the online submarket prefer to demand RSU to set a prefixed threshold of execution latency before allocating the AVs.
\item Asymmetric Information: There is asymmetric information among virtual simulators for their traffic and driving simulations. The traffic simulation (e.g., movement predictions) can induce immediate responses from users. In contrast, driving simulations (e.g., training the traffic sign recognition models of AVs) cannot be measured by virtual simulators immediately.
\end{itemize}
To ensure efficient allocation and pricing results, it is important to consider the potential impact of this correlation and to find ways to address asymmetric information and externalities in the market.

\section{Multi-task Enhanced Mechanism Design}\label{sec:mechanism}


To tackle the multi-task DT offloading problem in autonomous driving systems with traffic and driving simulations empowered by generative AI, we propose the multi-task enhanced second-score auction-based mechanism, named MTEPViSA, based on the EPViSA proposed in~\cite{xu2022epvisa}. Integrating the multi-dimensional auction~\cite{tang2018multi} and the enhanced second-price auction~\cite{arnosti2016adverse}, the MTEPViSA consists of four components in the mechanism, including the bidding process, the scoring rule, the allocation rules, and the pricing rules. 

The MTEPViSA allocates and prices the winning AV in the online submarket by calculating the scoring rule. Therefore, we first define the AIGC-empowered  scoring rule similar to~\cite{tang2017momd} as follows.
To address the inefficiency issue of the PViSA mechanism, we apply several advanced auction theory techniques in auction theory~\cite{che1993design, arnosti2016adverse} to enhance the auction-based  mechanism by overcoming the externalities in the online submarket and the asymmetric information in the offline submarket described in Subsection~\ref{sec:MTEPViSA}. For the online traffic simulation, AVs in the online submarket are allowed to submit their prices and preferred deadlines of DT tasks to the auctioneer. In addition, for the offline driving simulation, the auctioneer, e.g., the proxy of the RSUs, can determine the allocation rule according to the received bids from AVs with the AIGC scoring rule. Moreover, for the offline driving simulation, by adopting the price scaling factor $\alpha \geq 1$ in the offline submarket, the auctioneer can capture a significant fraction of the social surplus from both performance and brand virtual simulators. Finally, we analyze the properties of the MTEPViSA mechanism in Subsection~\ref{sec:MTEPViSAproperty}.

\subsection{Designing the MTEPViSA Mechanism}\label{sec:MTEPViSA}
This subsection describes the workflow and property analysis of the multi-task enhanced second-score auction-based mechanism.
To begin with, the definition of the multi-task DT scoring rule that is similar to~\cite{che1993design} is provided as follows.
		
\begin{definition}[Multi-task DT Scoring Rule] 
Let $b_1^\emph{DT}$ be any offered bidding price of AV $i$, the multi-task DT scoring rule $\Phi(b_i^\emph{DT},\mathbf{d}_{i})$ under deadline requirement $\mathbf{d}_{i}$ for each task $n=1,\ldots, N_i$ is defined as
\begin{equation}
    \Phi(b_i^\emph{DT}, \mathbf{d}_{i}) = b_i^\emph{DT} + \sum_{n=1}^{N_1}\phi(d_{i,n}),\label{eq:score}
\end{equation}
where $\mathbf{d}_{i} = (d_{i,1},\ldots,d_{i,N_i})$ contains the submitted deadlines of AV $i$'s DT tasks and $\phi(\cdot)$ is a non-decreasing function and $\phi(0)=0$.
\end{definition}

The scoring rule defined in Eq. (\ref{eq:score}) involves the deadlines of DT tasks and one element in the price vector. Therefore, for each AV $i$, $N$ scores are calculated. Based on these scores, the marginal score sequence $\chi_i = \{\chi_{i,1}, \chi_{i,2}, \ldots, \chi_{i,N_i}\}$ can be calculated for each AV $i$. The marginal score indicates the AV $i$'s score increase when the total number of executed tasks increases. In addition, the $n$ marginal score of AV $i$ can be defined as
\begin{equation}
    \chi_{i,n} = \begin{cases}
 \Phi(b_i^\emph{DT}, \mathbf{d}_{i,1}) &  n = 1,\\ 
 \Phi(b_i^\emph{DT}, \mathbf{d}_{i,n}) - \Phi(b_{i}^\emph{DT}, \mathbf{d}_{i, n-1})& 2\leq n \leq N_i,
\end{cases}
\end{equation}
where $\mathbf{d}_{i,n} = (d_{i,1},\ldots,d_{i,n})$.
Then, we have the assumption on the property of marginal scores as follows~\cite{tang2018multi}.
\begin{assumption}[Marginal Score]\label{as1}
For any AV $i\in\mathcal{I}$, the marginal score sequence $\chi_i$ is non-negative and non-increasing in $n$, i.e., $\chi_{i,n} \geq \chi_{i,n+1}, n=1,2,\ldots,N_i-1$.
\end{assumption}
The meaning of Assumption \ref{as1} is that performing additional simulations provides a higher score and the score is non-increasing with the performed simulations. 

The auctioneer can calculate the scoring rule based on previous transaction results and current submitted bids and deadlines. In the online submarket, AVs submit their multi-dimensional bids $\mathbf{b}^\emph{DT} = ((b_1^\emph{DT}, \dots , b_I^\emph{DT}), \mathbf{d} = (\mathbf{d}_1, \dots, \mathbf{d}_I))$ to the auctioneer. The auctioneer computes the scores $\Phi = \Phi(b^\emph{DT}, \mathbf{d}) = (\Phi_1(b_1^\emph{DT}, \mathbf{d}_1), \dots, \Phi_I(b_I^\emph{DT}, \mathbf{d}_I))$ to the auctioneer. Then, the auctioneer determines the winning AV in the online submarket for providing online simulations services according to the calculated scores. The auctioneer allocates the trader with the highest score as the winning physical entity, as follows:
\begin{equation}
    g_i^\emph{DT}(\Phi) = 1_{\{\Phi_i>\max \{\Phi_{-i}\}\}}.
\end{equation}
In addition, the payment that the winning AV needs to pay is the bidding price of the second highest score, i.e.,
\begin{equation}
    p_i^\emph{DT}(\Phi) = g_i^\emph{DT}(\Phi) \cdot b^\emph{DT}_{\arg\max \{\Phi_{-i}\}}.
\end{equation}

In the offline submarket, virtual simulators submit their bids $b^\emph{SIM} = (b^\emph{SIM}_0, b^\emph{SIM}_1, \dots, b^\emph{SIM}_K)$ to the auctioneer. In the MTEPViSA mechanism, the price scaling factor $\alpha \geq 1$ is utilized. First, the auctioneer determines the allocation probabilities for traffic virtual simulators as $g^\emph{SIM}_{k}(b^\emph{SIM})=1_{b^\emph{SIM}_{k} > \alpha b^\emph{SIM}_{-k}}$. Then, the allocation probability of the virtual simulator is calculated as $g^\emph{SIM}_0(b^\emph{SIM})\leq 1-\sum_{k=1}^{K}g^\emph{SIM}_k(b^\emph{SIM})$. Based on the price scaling factor $\alpha$, the winning virtual simulator is required to pay 
\begin{equation}
    p^\emph{SIM}_k(b^\emph{SIM}) = g_k^\emph{SIM}(b^\emph{SIM})\cdot \rho_k^\emph{SIM},
\end{equation}
where
\begin{equation}
    \rho_k^\emph{SIM} = \begin{cases}
        T^\emph{total}_{i,n,j,0}b^\emph{SIM}_{0}, &  k=0, \\ 
        T^\emph{total}_{i,n,j,k}\alpha \max  \{b^\emph{SIM}_{-k}\}, & k=1,\dots,K.
    \end{cases}
\end{equation}
By introducing the price scaling factor in the pricing rule in the offline submarket, the MTEPViSA mechanism can increase the expected social surplus of RSUs by providing offline simulations services compared with the traditional second-price auction. We then analyze the strategy-proofness of the MTEPViSA mechanism in Theorem~\ref{theorem1}. 

These allocation and pricing rules are effective and efficient when the efficient scoring rule exists~\cite{tang2018multi} and the price scaling factor is selected as $\alpha_\iota = \max{(1,\gamma\mathbb[Q_{\iota,0}]/\mathbb{E}[Q_{\iota,(2)}])}$~\cite{arnosti2016adverse}, where $\iota$ is the winning AV in the online submarket. 
Finally, under the cost-per-time payment model of traffic and driving simulations and the efficient multi-task DT scoring rule, the MTEPViSA is fully strategy-proof and adverse-selection-free.
\subsection{Property Analysis}\label{sec:MTEPViSAproperty}

To maximize its utility, each AV $i$ can choose the deadlines that can maximize its valuation $v_i$ and externalities $\phi(d_{i,n})$ for the offline submarket. In Proposition~\ref{proposition2}, each AV can choose the optimal set of deadlines to maximize its expected payoff.
\begin{proposition} [Optimal Deadline]\label{proposition2}
    The optimal deadline bidding strategy for task $n$ of the AV $i$ is given by
    \begin{equation}
        d_{i,n}^* = \arg\max_{d\in(0,d_{i,n}]} (v_{i} + \sum_{n=1}^{N_i}\phi(d)).\label{eq:deadline}
    \end{equation}
\end{proposition}
\begin{proof}
For any given multi-dimensional bid $(\bar{b}_i^\emph{DT}, \bar{d}_i)$ of AV $i$, there is always another bid $(\hat{b}_i^\emph{DT}, \hat{d}_i)$ that can be made to an expected utility for AV $i$. This results in an expected utility for the physical bidder $i$ that is at least as high as the utility obtained from the original bid $(\bar{b}_i^\emph{DT},\bar{\mathbf{d}}_i)$.
First, the deadline $\hat{\mathbf{d}}_i$ can be obtained from Eq. (\ref{eq:deadline}). Second, the deadline for the new bid $\hat{b}_i^\emph{DT}$ can be determined by $\Phi(\hat{b}_i^\emph{DT}, \hat{\mathbf{d}}_i) = \Phi(\bar{b}_i^\emph{DT},\bar{\mathbf{d}}_i)$, indicating that both bids result in the same score and allocation probability. If the bidder loses, their utility will be zero if they achieve this score. Moreover, if the bidder wins, the new bid $(\hat{b}_i^\emph{DT},\hat{\mathbf{d}}_i)$ will yield the utility higher than or equal to the utility obtained from submitting other bids, i.e., 
\begin{equation}
    \begin{aligned}
        v_i - \hat{b}_i^{\tmop{DT}} - &(\max \{ \Phi_{\mathcal{I} / \{ i \}} \} + \sum_{n=1}^{N_i} \phi
        (\hat{d}_{i,n})) \geq \\&v_i - \bar{b}_i^{\tmop{DT}} - (\max \{
        \Phi_{\mathcal{I} / \{ i \}} \} + \sum_{n=1}^{N_i}\phi (\bar{d} _{i,n})),
    \end{aligned}
\end{equation}
where $\Phi_{\mathcal{I} / \{ i \}}$ consists of the scores $(\Phi(b_1^\emph{DT}, \mathbf{d}_1),\ldots,\Phi(b_{i-1}^\emph{DT}, \mathbf{d}_{i-1}), \Phi(b_{i+1}^\emph{DT}, \mathbf{d}_{i+1}), \ldots, \Phi(b_I^\emph{DT}, \mathbf{d}_I) )$
This proposition holds true because the deadline is determined through the calculation of Eq. (\ref{eq:deadline}).
\end{proof}
For the optimality of the selection of quality, a similar proof of Proposition \ref{proposition2} can be found in~\cite{tang2018multi}. Based on the bids submitted by AVs and the chosen optimal deadline, the auctioneer can maintain an efficient  scoring rule, which can maximize social surplus, to guide the allocation decisions in the online submarket, as follows.
Then, the efficient multi-task DT scoring rule can be defined as follows.
\begin{definition}[Efficient Multi-task DT Scoring Rule]\label{definition2} An efficient multi-task DT scoring rule can be expressed as
\begin{equation}
        \Phi(b^\emph{DT}, \mathbf{d}^*) = b^\emph{DT} + T(\mathbf{d}^*)[\gamma S_D^\emph{SIM}(\mathcal{M}) + S_\emph{T}^\emph{SIM}(\mathcal{M})],
    \end{equation}
where $T(\mathbf{d}^*) [\gamma S_D^\emph{SIM}(\mathcal{M}) + S_\emph{T}^\emph{SIM}(\mathcal{M})]$ is the social surplus of virtual simulators by providing simulations and $T(\mathbf{d}^*)$ is the realized duration of AV training.
\end{definition}

For a mechanism, strategy-proofness indicates that participants will not get a higher utility by changing their truthful bids. A mechanism is considered strategy-proof if and only if it can be described by a critical payment function $\psi$, such that a bidder $n$ is deemed the winner if and only if their bid $b_n$ exceeds the threshold price $\psi (b_{-n})$ when compared to the other competing bids $b_{-n}$. Once bidder $n$ has won the auction, the payment charged by the auctioneer is the critical payment $\psi$.  Adverse-selection free indicates that if the existence of market externalities and asymmetric information is independent of bidders' valuations, then under this mechanism, the factors of market externalities and asymmetric information are also independent of the allocation rules of the  mechanisms. As a consequence, it should be highlighted that the MTEPViSA mechanism is fully strategy-proof and adverse-selection free, as demonstrated in the following theorem.
\begin{theorem}\label{theorem1}
    The MTEPViSA mechanism is fully strategy-proof and adverse-selection-free in the market with the efficient multi-task DT scoring rule and the cost-per-time model of simulations.
\end{theorem}

\begin{proof}
    To demonstrate that the MTEPViSA mechanism is fully strategy-proof, we must identify the critical payment functions for traders in both the online and offline submarkets to satisfy the conditions for a strategy-proof auction. To begin with, we show that there is a critical payment function $\psi^{on}(b_{-i}^\emph{DT})$ for the MTEPViSA mechanism in the online submarket. If a bidder $i$ in the online submarket submits a truthful bid, their score can be determined by the function $\Phi(b^{DT}_i, \mathbf{d}_i)$, given any deadlines $\mathbf{d}_i$. It is necessary to demonstrate that bidder $i$ cannot increase their benefit by altering their bid. If the bidder were to submit a false bid $b_i' \neq b^\emph{DT}_i$, and did not win the auction, their reward would be zero, regardless of the specified deadline $\mathbf{d}_i$ or the score calculation function $\Phi(b^{DT}_i, \mathbf{d}_i)$. However, if bidder $i$ were to win the auction by submitting a false bid, their expected reward can be represented as follows:
    \begin{equation}
    \begin{aligned}
        S^{\tmop{DT}}_i & = v_i - b_i'\\
        & = v_i - (\max \{ \Phi_{- i} \} + \sum_{n=1}^{N_i}\phi (d_{i,n}))\\
        & = \Phi_i - \max \{ \Phi_{- i} \},
    \end{aligned}
\end{equation}
where $\max{\Phi_{-i}}$ represents the highest score excluding the bid of bidder $i$. Hence, regardless of whether the bidder wins or loses under either $\Phi'$ or $\Phi_i$, the utilities that they receive will always be lower than  or equal to the utility that they would receive if they submitted the truthful bid. The critical payment function in the online submarket can be represented as $\psi_{on} (b^{\tmop{DT}}_i, \mathbf{d}_i ) = \max \{ \Phi_{-i} \} + \sum_{n=1}^{N_i}\phi (d_{i,n})$. Additionally, the auctioneer must compute the synchronization scores for bidders in the online submarket. As a result, all bidders in the online submarket are protected against false-name bidding.

The critical payment function for the MTEPViSA mechanism in the offline submarket is $\psi_{\tmop{off}} (b^{\tmop{SIM}}_{-k}) = \alpha \max \{ b^{\tmop{SIM}}_{-k} \}$, where $\alpha \geq 1$. With this critical payment function in place, the top-performing bidder can win by ensuring that $\psi_{\tmop{off}} (b^{\tmop{SIM}}_{-k}) \geq \max \{ b^{\tmop{SIM}}_{-k} \}$. Furthermore, the mechanism is proof against false-name bidding in the offline submarket if $\psi_{\tmop{off}} (b^{\tmop{SIM}}_{-k}) = \psi_{\tmop{off}} (\max \{ b^{\tmop{SIM}}_{-k} \})$. Consider a set of bids $b_{k}^{SIM}$ that result in $\psi_{\tmop{off}} (b^{\tmop{SIM}}_{-k}) \neq \psi_{\tmop{off}} (\max\{ b^{\tmop{SIM}}_{-k}\})$, and let there be two bidders in the offline submarket. If one bidder has a higher valuation than $\psi_{\tmop{off}} (b_{-k})$ and the other has a valuation of $\max \{ b^{\tmop{SIM}}_{-k} \}$, and $\psi_{\tmop{off}} (b^{\tmop{SIM}}_{-k}) < \psi_{\tmop{off}} (\max \{ b^{\tmop{SIM}}_{-k} \})$, then the first bidder could submit a lower price while keeping the other bids in the set $b_{-k}$. This means that the mechanism is not winner false-name proof. Conversely, if $\psi_{\tmop{off}} (b^{\tmop{SIM}}_{-k}) > \psi_{\tmop{off}} (\max \{ b^{\tmop{SIM}}_{-k} \})$, the losing bidder in the offline submarket could submit a higher bid compared to the winner's bid while maintaining the other bids in the set $b^{\tmop{SIM}}_{-k}$. As a result, the mechanism in the offline submarket is loser false-name proof.
  \begin{figure*}[t]
\vspace{-0.8cm}
    \centering
    \subfigure[Trajectories 1, R2 score=0.9972.]{\includegraphics[width=0.24\linewidth]{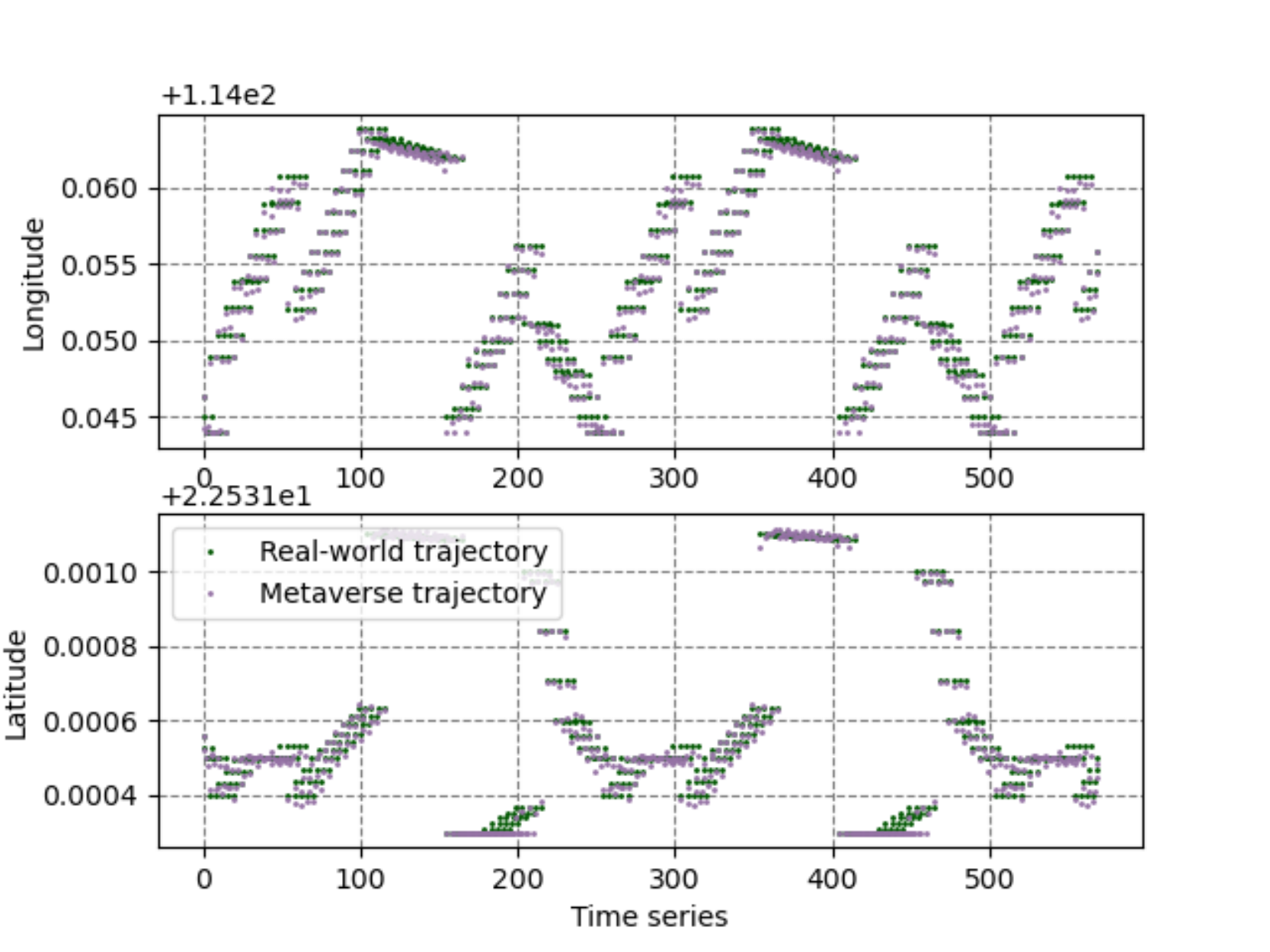}%
        \label{case1}}
    \subfigure[Trajectories 2, R2 score=0.9939.]{\includegraphics[width=0.24\linewidth]{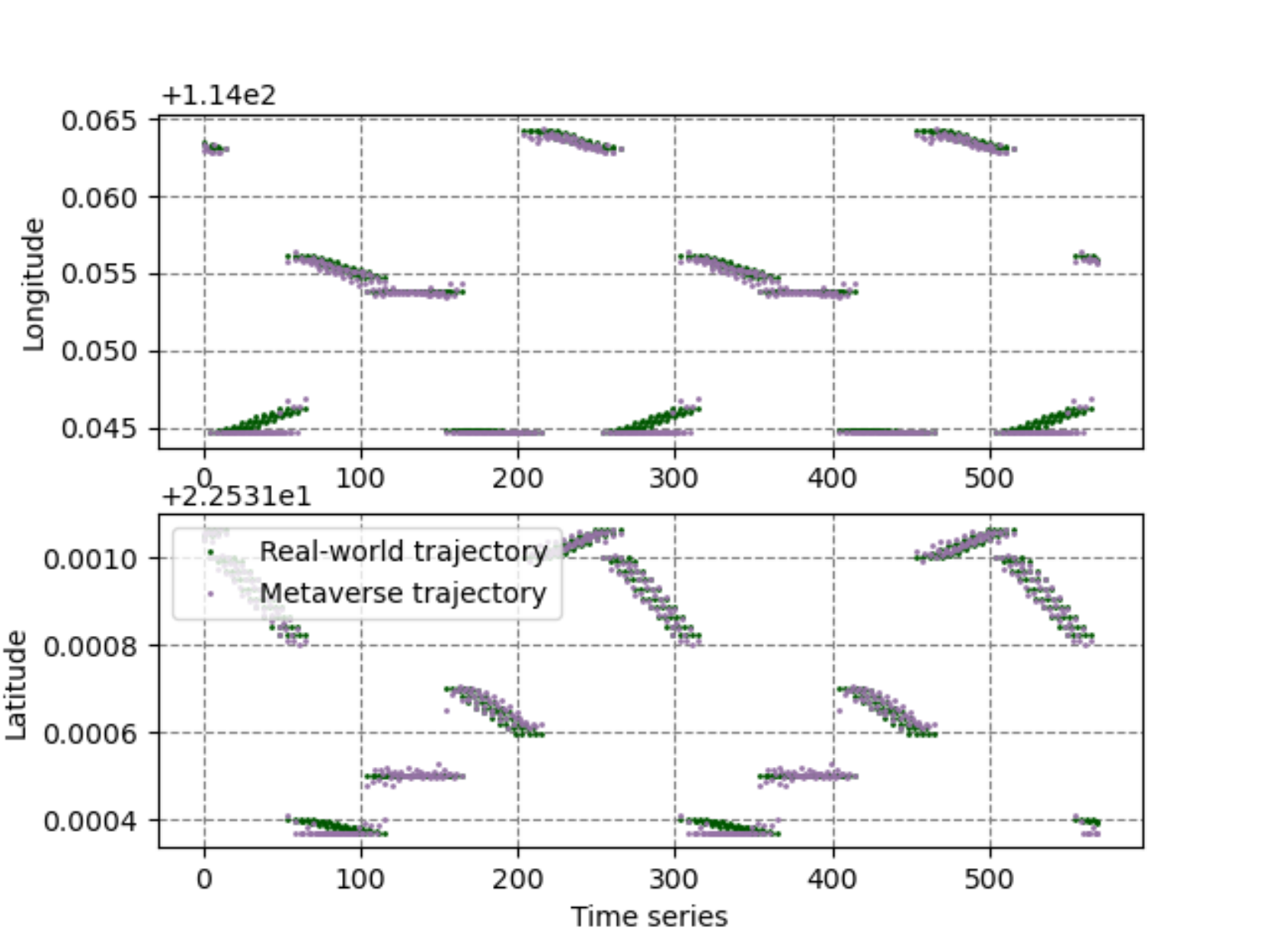}%
        \label{case2}}
    \subfigure[Trajectories 3, R2 score=0.9984.]{\includegraphics[width=0.24\linewidth]{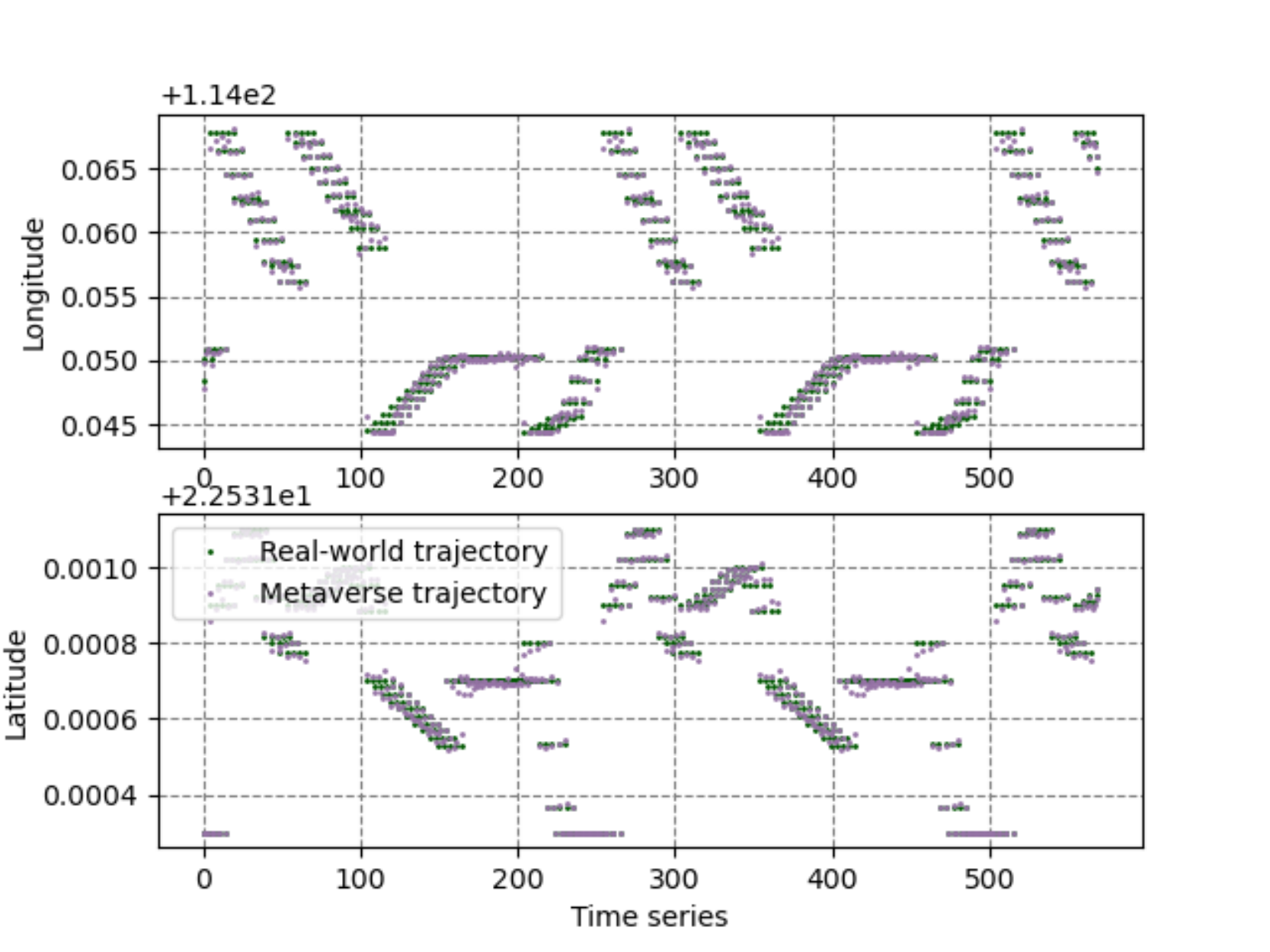}%
        \label{case3}}
    \subfigure[Trajectories 4, R2 score=0.9983.]{\includegraphics[width=0.24\linewidth]{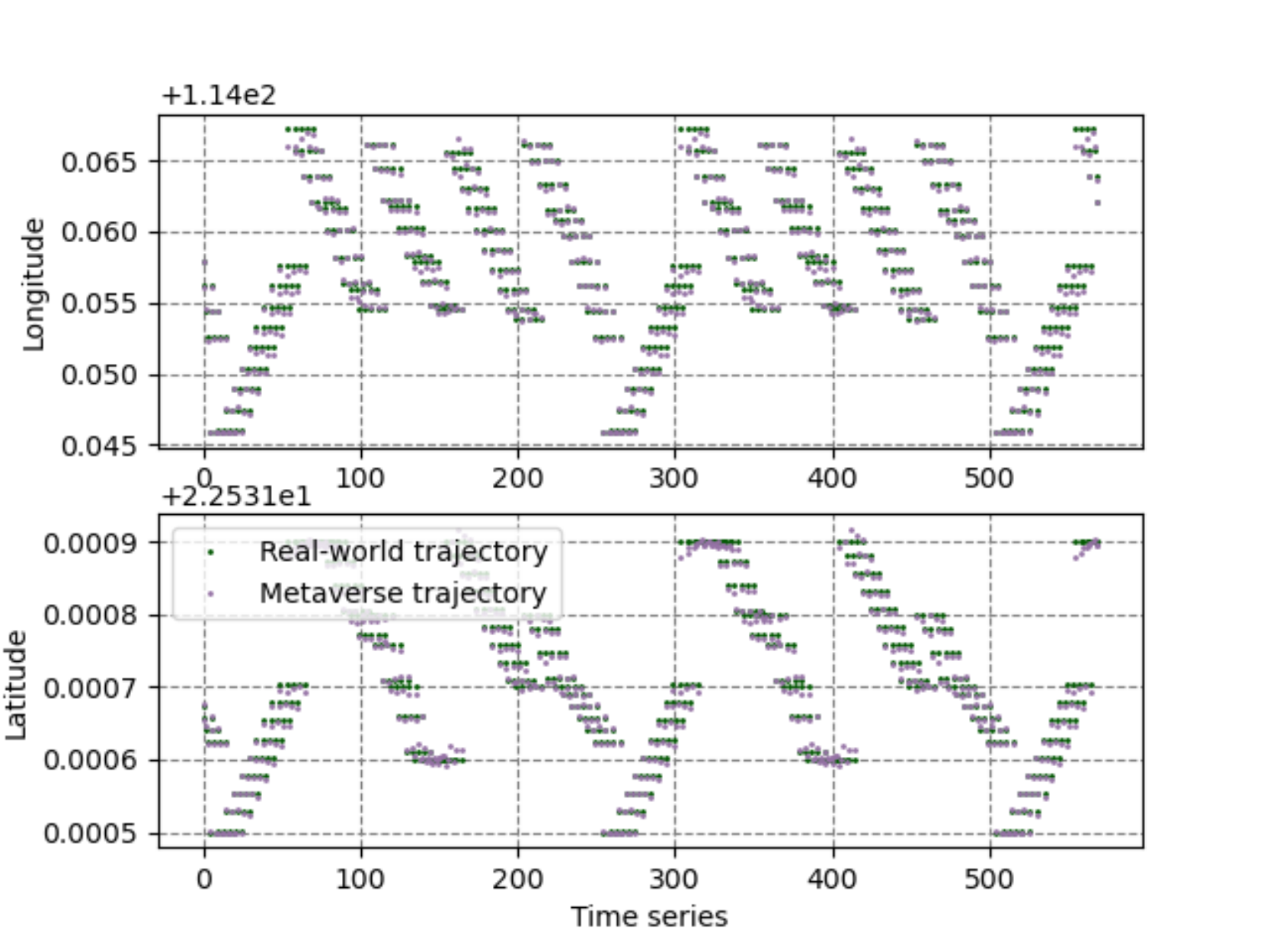}%
        \label{case4}}
    \caption{Difference between the real-world trajectories and DT-assisted predicted trajectories.}
    \label{fig:trajectory}
\end{figure*}

To show that the mechanism is free from adverse selection, the critical payment function in the online submarket is quasi-linear and the critical payment function in the offline submarket is homogeneous of degree one~\cite{arnosti2016adverse}. In the online submarket, we consider two types of external effects from the offline submarket, $d \in \{ 0, \infty \}$, with a probability of $\Pr (\phi = 0) \in (0, 1)$ while keeping the other bidding prices $v_{- i}$ constant. If $\phi = 0$, there are no external effects from the offline submarket, and we have $g_{i }^{\tmop{DT}} (v + \phi (0) ) = g_{i }^{\tmop{DT}} (v) = 1_{\{v_i > \max v_{- i}\}} = 1_{\{v_i > \psi_{\tmop{on}} (v_{- i}, 0)\}}$. If $\phi = \infty$, then $g_{i}^{\tmop{DT}} (v + \phi (0) ) = g^{\tmop{DT}}{i } (v + \infty) = 1_{\{v_i > \psi_{\tmop{on}} (v_{- i}, \phi(\infty))\}} = 0$, meaning no bidder can win in the online submarket. Thus, the proposed mechanism in the online submarket is free from adverse selection. In the offline submarket, suppose that $v \in { 1, c }$ with a probability of $\Pr (C = 1) \in (0, 1)$ while keeping the vehicular MR Metaverse simulation qualities constant. It can be shown that $g^{\tmop{SIM}}_{0 } (v m) = 1_{\{ v = c \}}$. When $v = 1$, $g^\emph{SIM}_k (v m) = g^\emph{SIM}_k (v m) = 1_{\{ m_i > \psi_{\tmop{off}} (m_{- i}) \}} = 1$, and therefore $g^{\tmop{SIM}}_0 (\tmop{cm}) = 0$. When $v = c, g^\emph{SIM}_k (\tmop{cm}) = 1_{\{ c m_i > \phi_{\tmop{off}} (c m_{- i}) \}} = 0$, meaning no top-performing bidder can win the auction, and then $g^\emph{SIM}_0 (\tmop{cm}) = 1$.
In conclusion, we have shown that the MTEPViSA mechanism is both strategy-proof for bidders in the online and offline submarkets and free from adverse selection by utilizing the efficient multi-task DT scoring rule and cost-per-time model.
\end{proof}

From Theorem~\ref{theorem1}, we can conclude that the proposed mechanism for AVs and virtual simulators is fully strategy-proof. That is, AVs and virtual simulators cannot manipulate their bids to achieve higher utility. Although we introduce the scoring rule and price scaling factors to eliminate the externalities and asymmetric information, these additional components may not provide additional information to AVs and virtual simulators when they develop their own strategy. In the MTEPViSA mechanism, the optimal strategy for AVs in the online submarket and virtual simulators in the offline submarket is to tell the truth. Moreover, due to the interaction of two submarkets leading to externalities and asymmetric information for traders, the proposed mechanism is free from adverse selection as all participants have sufficient motivation to join the market. Therefore, the social surplus achieved by the MTEPViSA mechanism is still efficient enough to avoid market failure.

Finally, we consider the implementation overhead of the proposed auction-based mechanism, where a centralized auctioneer collects bids, computes scoring rules, and determines allocation and pricing results. To begin with, $I$ AVs and $K$ virtual simulators submit their bids to the auctioneers. Let $N$ be the number of DT tasks of each AV, the computation complexity to compute the scores is $O(IK\log(K)(N))$. Then, the computation complexity to sort the scores of AVs in the online market is $O(I\log(I))$. Finally, the computation complexity of the determination and pricing of the winning virtual simulator is $O(K)$. Overall, the computation complexity of the proposed MTEPViSA is $O(IKN\log(K) + I\log(I)+ K)$.

        
  
\section{Experimental Results}\label{sec:exp}
\begin{figure}[t]
    \centering
    \includegraphics[width=1\linewidth]{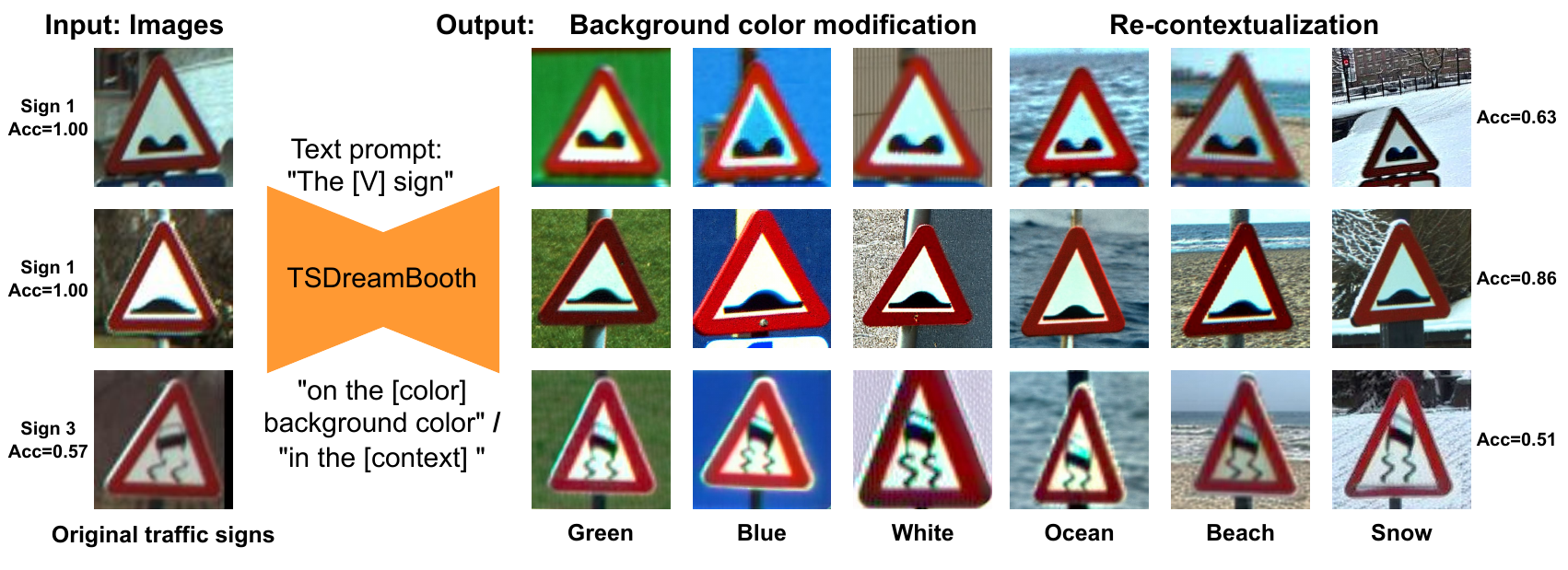}
    \caption{Synthesized traffic signs generated by TSDreambooth for background color modification and re-contextualization.}
    \label{fig:AIGC}
\end{figure}

\begin{figure*}[t]
    \centering
    \includegraphics[width=1\linewidth]{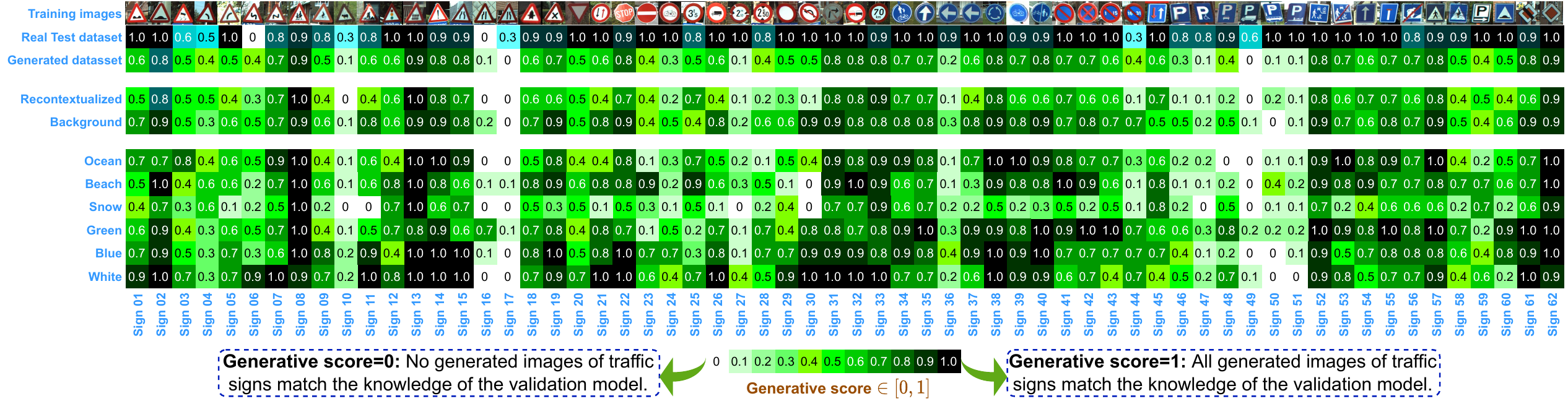}
    \caption{The generative score of the TSBreamBooth fine-tuned on the BelgiumTS dataset.}
    \label{fig:gs}
\end{figure*}

In this section, we implement the generative AI-empowered autonomous driving simulation system for the vehicular MR Metaverse and the proposed mechanism. First, we demonstrate the performance of the DT-assisted movement prediction model in Subsection~\ref{sec:expdt} and the generative AI-empowered traffic and driving simulation model in Subsection~\ref{sec:expar}. Then, we evaluate the performance of the proposed mechanism under different market parameters and system settings in Subsection~\ref{sec:expauction}.
\subsection{Experimental Setups}\label{sec:expsetup}
\begin{figure*}[t]
\vspace{-0.5cm}
    \centering
    \subfigure[Social surplus v.s. number of AVs.]{\includegraphics[width=0.32\linewidth]{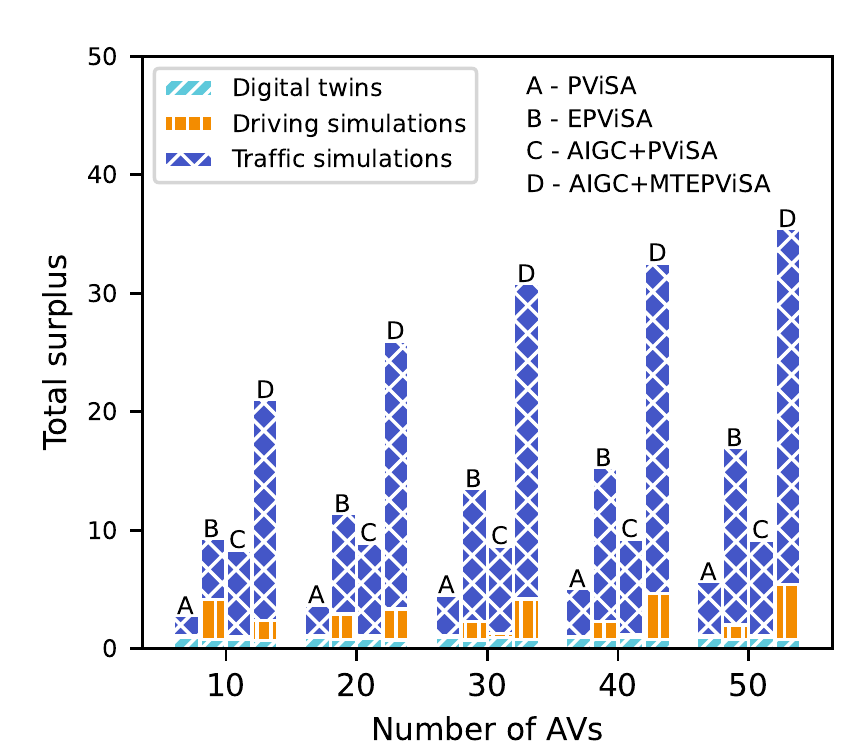}%
        \label{fig:av}}
    \subfigure[Social surplus v.s. number of traffic simulators.]{\includegraphics[width=0.32\linewidth]{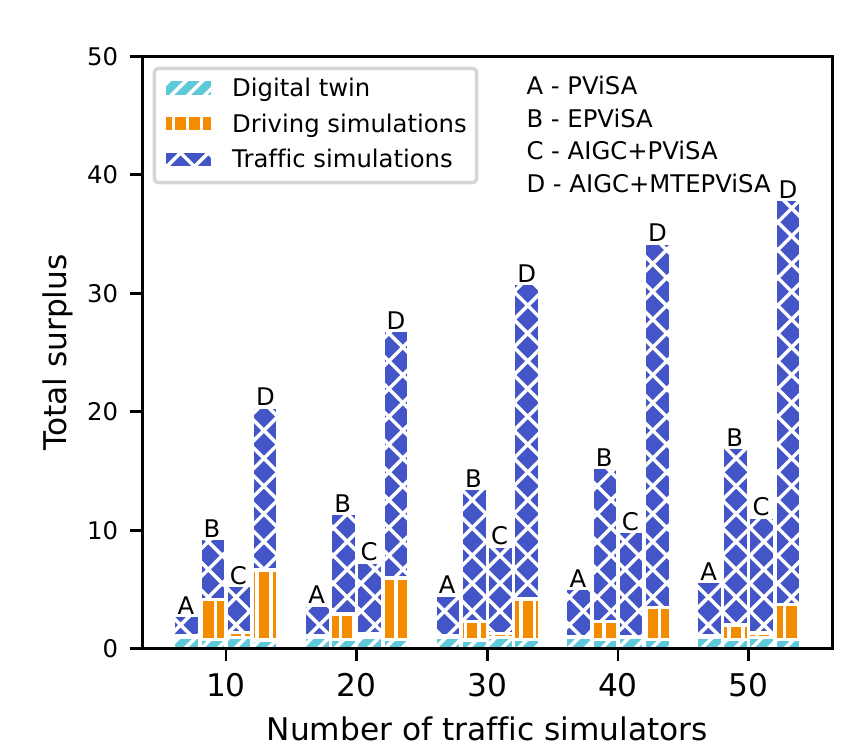}%
        \label{fig:ar}}
    \subfigure[Social surplus v.s.  generative score.]{\includegraphics[width=0.32\linewidth]{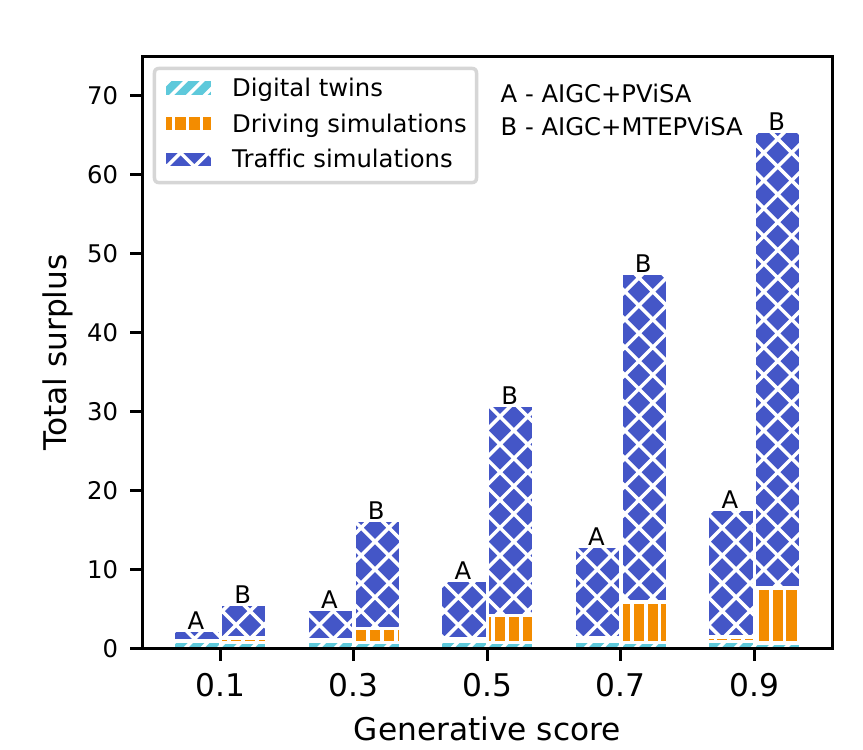}%
        \label{fig:sa}}
    \caption{Performance evaluation under different sizes of  the market and generative scores.}
    \label{fig:sw}
\end{figure*}
In the simulation of the vehicular MR Metaverse, we consider an autonomous driving system with 30 AVs, 30 virtual simulators, and 1 RSU by default. For each RSU, 20 MHz uplink and 20 MHz downlink channels are allocated for DT task uploading and AR recommendation streaming, respectively. In addition, the CPU frequency of RSU is set to 3.6 GHz, and the GPU frequency is set to 19 GHz. The channel gain between RSUs and AVs is randomly sampled from $U[0, 1]$, where $U$ denotes the uniform distribution. The transmit power of AVs is randomly sampled from $U[0, 1]$ mW and the transmit power of RSUs is randomly sampled from $U[0, 5]$ mW. The additive white Gaussian noise at AVs and RSUs is randomly sampled from $\mathcal{N}(0, 1)$, where $\mathcal{N}$ denotes the normal distribution. For each DT task generated by AV, the data size is randomly sampled from $U[0, 0.5]$ MB,  the required CPU cycles per unit data are randomly sampled from $[0, 2]$ Gcycles/MB, and the required deadline is randomly sampled from $U[1, 1.5]$ seconds. For each simulation, the data size is randomly sampled from $U[0, 2.5]$ MB and the required GPU cycles per unit data are randomly sampled from $U[0, 5]$ Gcycles/MB. The valuation of AVs for accomplishing the DT tasks is randomly sampled from $U[0, 1]$ and the number of preferences of AVs is sampled from $Zipf(2)$, where $Zipf$ denotes the Zipf distribution. The relative bargaining power of the offline virtual simulator  is set to 1 while the default  accuracy is 0.5. For digital twin-assisted vehicular movement prediction, we set the past $P$ steps to 60 and the future $F$ steps to 5. In addition, we set epoch $e$ to 500, batch size to 40, and dropout to 0.05. The default local relative accuracy is set to 0.53~\cite{tran2019federated} and the default generative score is sampled from $U[0.4, 0.6]$.

The simulation environment for the vehicular MR Metaverse was created using a 3D model of a few city blocks in New York City. Geopipe, Inc. developed the model by utilizing AI to build a digital replica from photographs captured throughout the city. The simulation involves an autonomous car navigating through a road, surrounded by artificially placed highway advertisements. Eye-tracking data was collected from human participants who were immersed in the simulation using the HMD Eyes addon from Pupil Labs. Following the simulation, the participants completed a survey to assess their subjective opinion level of interest in each simulation.

\subsubsection{Digital Twin-assisted Vehicular Movement Prediction}\label{sec:expdt}
Through continuously updating DTs in the virtual space, AVs can leverage the results of online traffic simulations for improving driving safety and traffic efficiency. Specifically, we use the historical trajectory data of AVs in DT to predict their future movements to make the concept of DT-assisted autonomous driving more concrete. Let the location of $i$ at time slot $t$ be $p_i^t = (x_i^t,y_i^t)$, where $x_i^t$ and $y_i^t$ are longitude and latitude of AV $i$, respectively. The historical trajectory of AV $i$ consists of the last $P$ locations can be represented as $\tau_i^\emph{past}(t) = (p_i^{t-P},\ldots,p_i^{t-1},p_i^t)$. When RSUs leverage AI models to predict the future movement of AVs that can be represented as $\mathcal{A}_j$ for RSU $j$. Then, the past trajectories input into the AI model of RSU $j$ predicts the movement $\tau_i^\emph{pre}(t) = \mathcal{A}_j(\tau_i^\emph{past}(t)) = (p_i^{t+1},p_i^{t+2},\ldots, p_i^{t+F})$ in the future $F$ steps of the vehicles and simulate the movements in the virtual space. In the training module, the AI model is evaluated by the mean squared error (MSE), i.e., the training loss is calculated as $\mathbb{E}_{\tau_i^\emph{past}(t), \tau_i^\emph{true}\sim \emph{DT}_i}(\tau_i^\emph{pre}(t)-\tau_i^\emph{true}(t))^2$. Finally, the running performance is evaluated by the R2 score $\mathcal{R}_{i,j}$, which is 1 when the predicted movements are perfectly correlated with the true movements.

\begin{figure*}[t]
\vspace{-0.5cm}
    \centering
    \subfigure[Total surplus v.s. number of tasks.]{\includegraphics[width=0.24\linewidth]{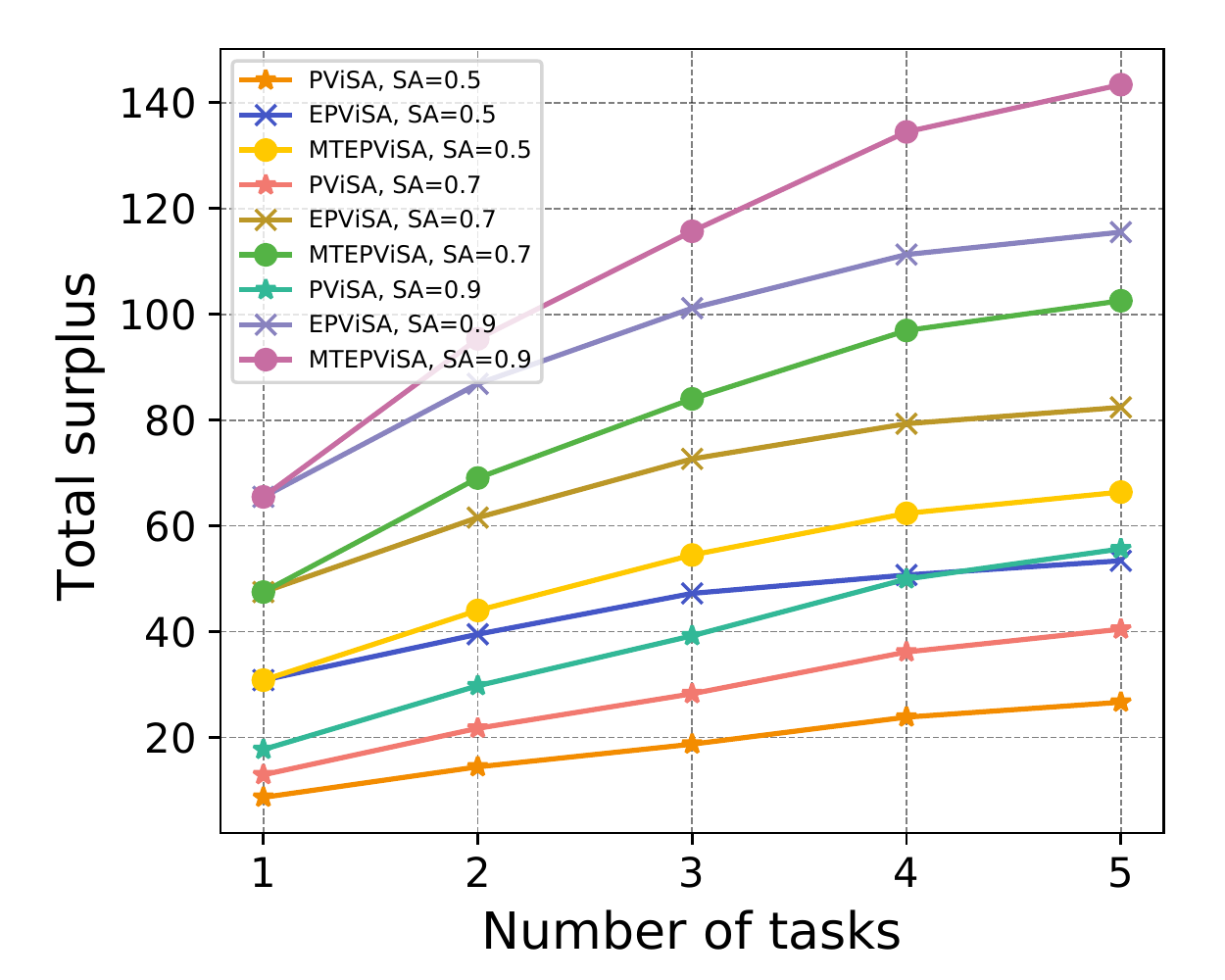}%
        \label{fig:revenuetotal}}
    \subfigure[DT surplus v.s. number of tasks.]{\includegraphics[width=0.24\linewidth]{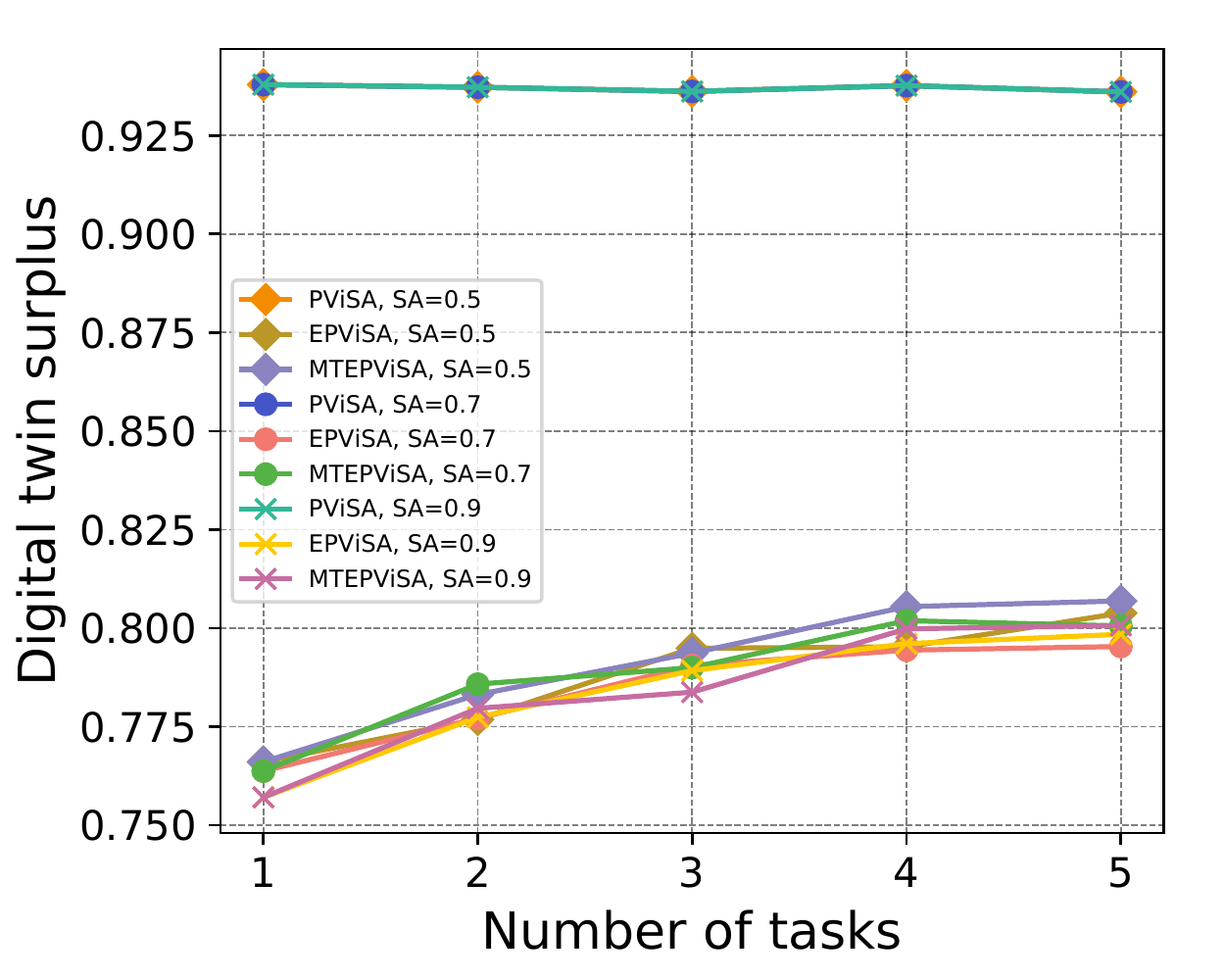}%
        \label{fig:revenuedt}}
    \subfigure[Traffic simulation surplus v.s. number of tasks.]{\includegraphics[width=0.24\linewidth]{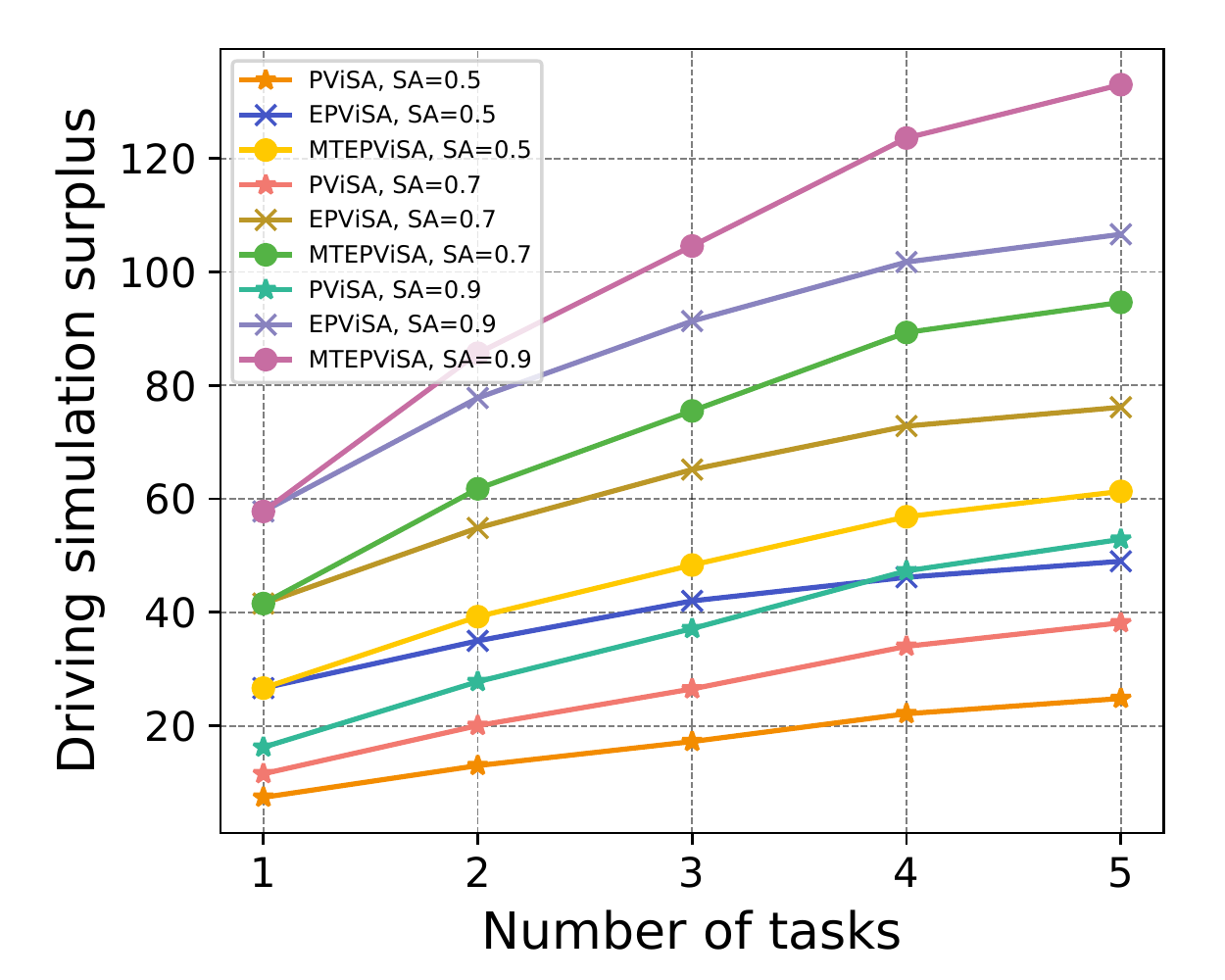}%
        \label{fig:revenuear1}}
    \subfigure[Driving simulation surplus v.s. number of tasks.]{\includegraphics[width=0.24\linewidth]{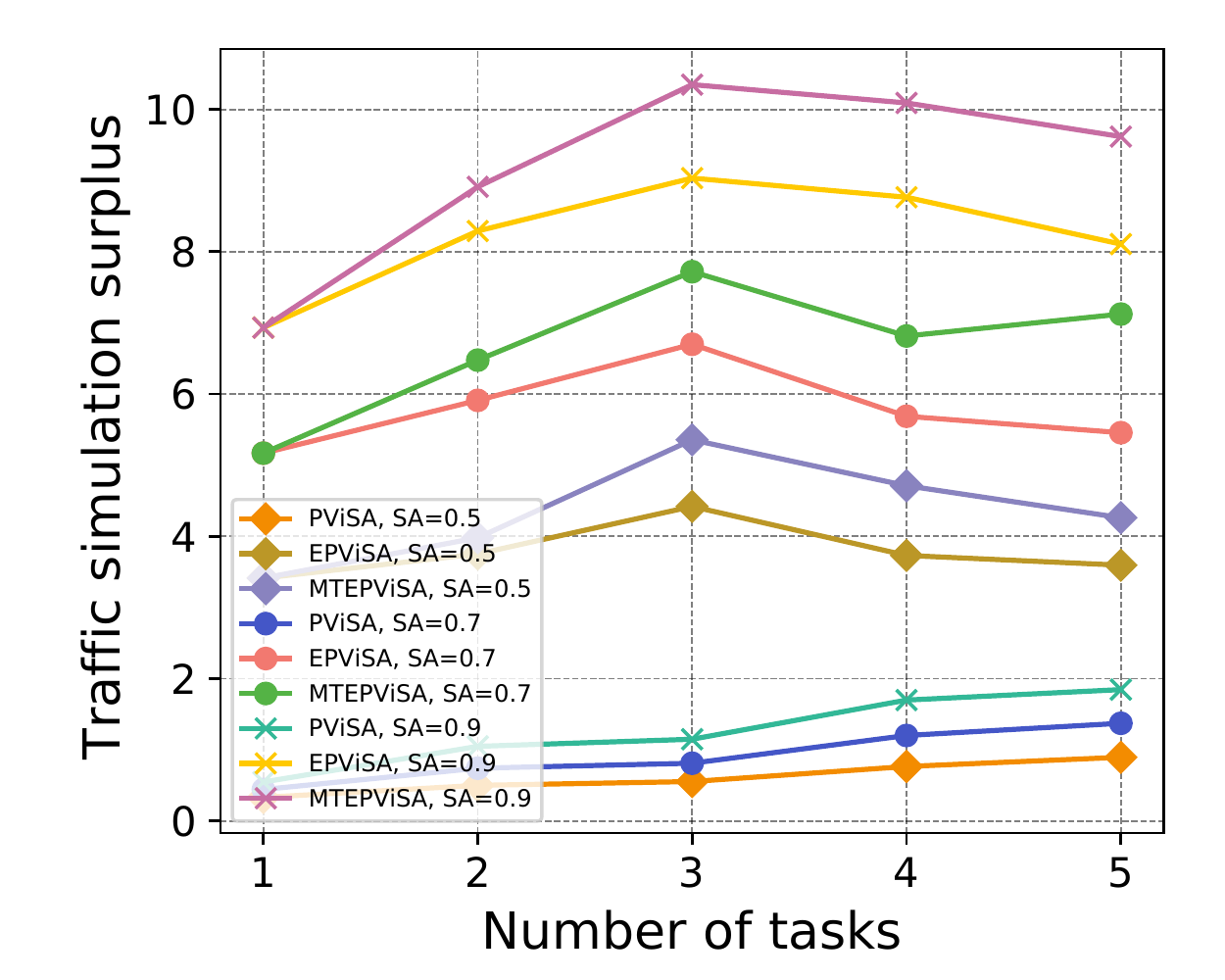}%
    \label{fig:revenuear2}}
    \caption{Performance evaluation of simulated experiments under different generative scores and numbers of tasks.}
    \label{fig:simulationrevenue}
\end{figure*}

As shown in Fig.~\ref{fig:trajectory}, we use four trajectory collections sampled from the dataset in~\cite{wang2019experience} to demonstrate the effectiveness of movement prediction of AVs based on the current location and historical routes. We select the LSTM model as the movement prediction model. We use the R2 score to evaluate the performance of the movement prediction AI model. R2 is commonly understood as using the mean as the error benchmark to see if the prediction error is greater or less than the mean benchmark error. When R2 score = 1, the predicted and true values in the sample are exactly equal without any error, indicating that the better the independent variable explains the dependent variable in the regression analysis. Based on this experiment, we collect the results of the prediction in the system to calculate the simulation accuracy in the virtual space. 


  
\subsection{Traffic Sign Synthesizing in Generative AI-empowered Simulation}\label{sec:expar}


Generative AI based on large text-to-image models, such as stable diffusion~\cite{rombach2022high} and Dreambooth~\cite{ruiz2022dreambooth}, will have a game-changing impact on content creation in the MR Metaverse. Dreambooth is a personalized diffusion model that learns to preserve the features of the specific subject and then generates new images based on this subject. To demonstrate the ability to generate diverse and high-quality images for the vehicular MR Metaverse. As illustrated in Fig.~\ref{fig:AIGC}, we experiment with modifying background color and re-contextualization for traffic signs, which is the iconic task for transportation systems.
\begin{figure*}[t]
\vspace{-0.5cm}
    \centering
    \subfigure[Total surplus v.s. number of tasks.]{\includegraphics[width=0.24\linewidth]{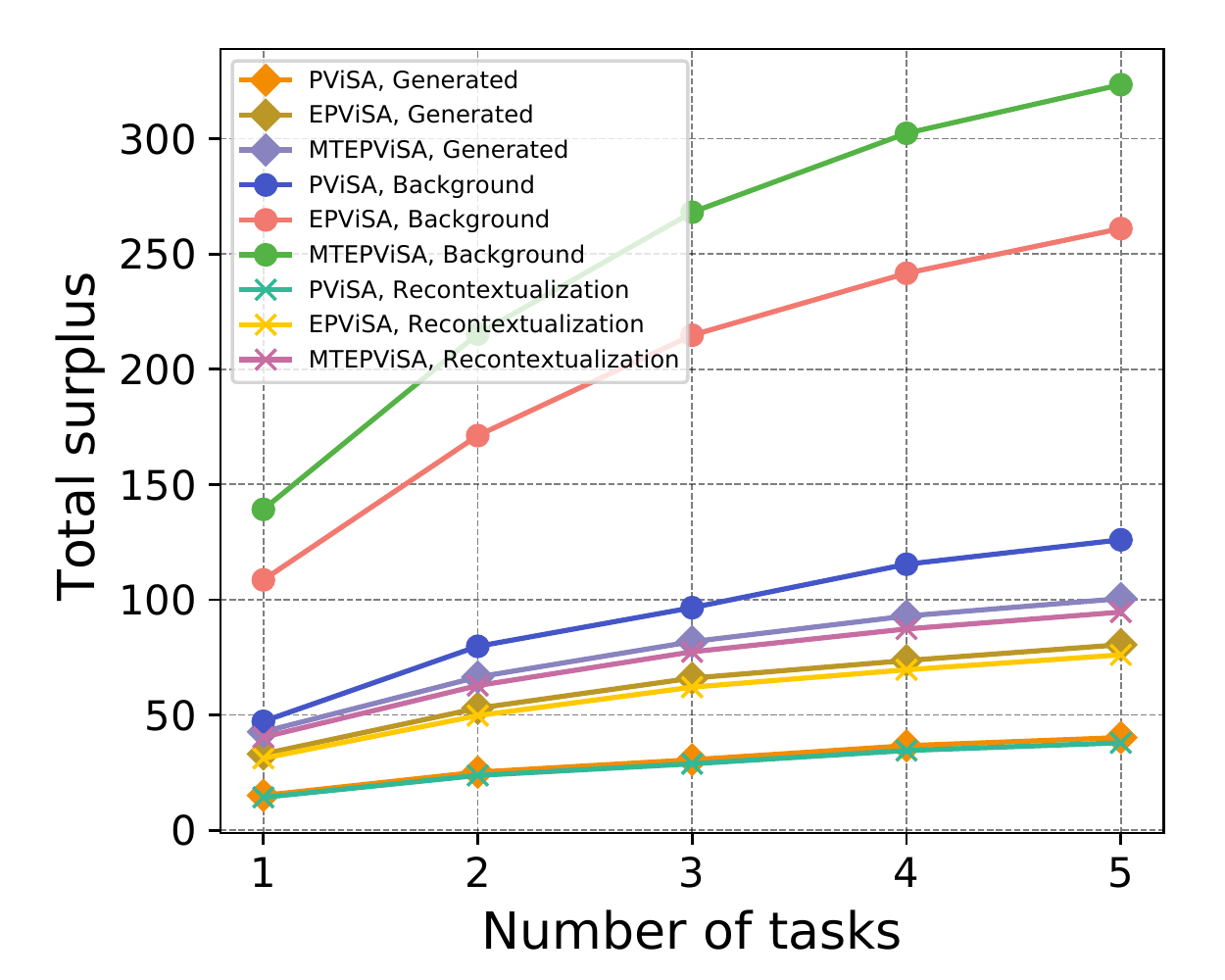}%
        \label{fig:real_revenuetotal}}
    \subfigure[DT surplus v.s. number of tasks.]{\includegraphics[width=0.24\linewidth]{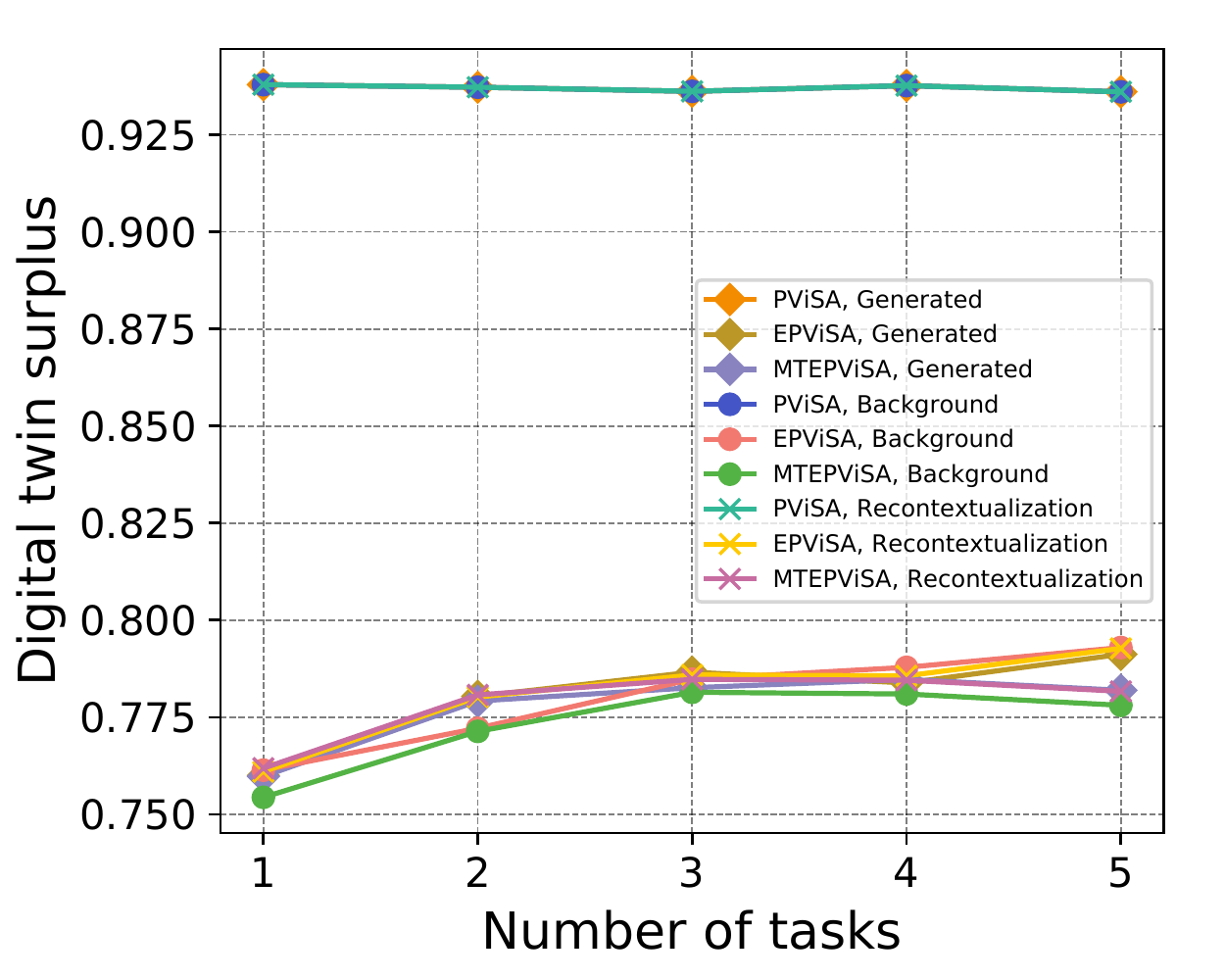}%
        \label{fig:real_revenuedt}}
    \subfigure[Traffic simulation surplus v.s. number of tasks.]{\includegraphics[width=0.24\linewidth]{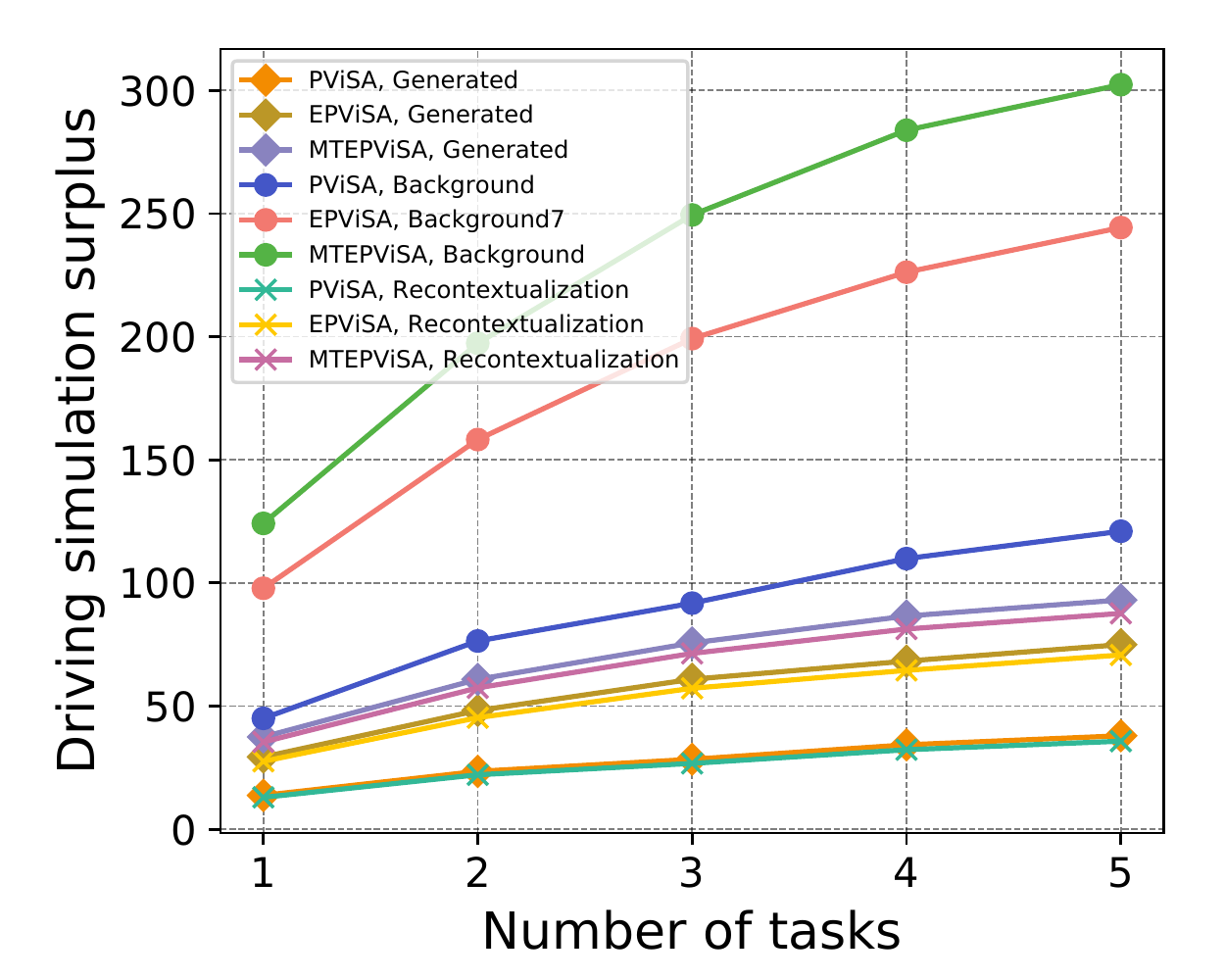}%
        \label{fig:real_revenuear1}}
    \subfigure[Driving simulation surplus v.s. number of tasks.]{\includegraphics[width=0.24\linewidth]{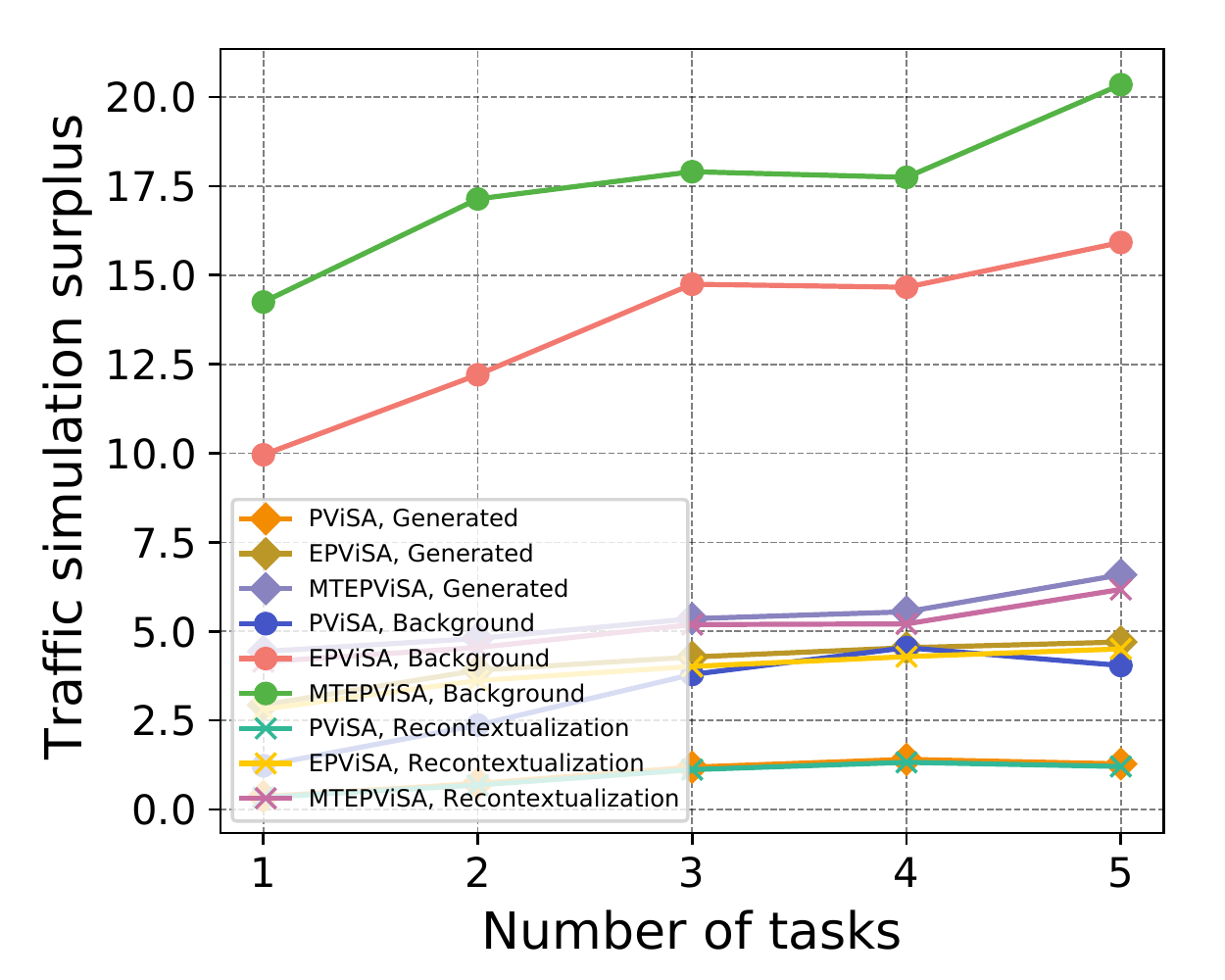}%
    \label{fig:real_revenuear2}}
    \caption{Performance evaluation of experiments under different datasets generated by TSDreamBooth and numbers of tasks.}
    \label{fig:real_simulationrevenue}
\end{figure*}
We first use the training set in BelgiumTS dataset~\cite{mathias2013traffic} to fine-tune the Dreambooth to the TSDreambooth. Then, we train a validation model based on the pre-trained GoogLeNet to fit the BelgiumTS dataset. The learning rate of the validation model is set to 0.1 and the number of epochs is set to 10. After optimizing via cross-entropy loss, the final recognition accuracy is 96\% on testing sets. Finally, we generate new images based on the testing set in BelgiumTS and evaluate the generative score using the validation model. We summarize the obtained generative score in Fig.~\ref{fig:gs} from the above experiments. As we can observe, the validation model performs almost perfectly in the real test dataset. However, for the generated dataset, the validation model can only recognize around 80\% of the images generated by TSDreambooth. Then, the synthesized datasets are leveraged for fine-tuning the validation model. During the fine-tuning, the batch size is set to 30 for one iteration. Finally, we obtain the local relative accuracy for the whole generated dataset ($\theta=0.82$), the background modification dataset ($\theta=0.42$), and the recontextualization dataset ($\theta=0.85$).

\subsection{Performance Evaluation of Auction-based Mechanisms}\label{sec:expauction}

Then, we evaluate the performance of the proposed mechanism under different system settings compared with the PViSA and the EPViSA proposed in~\cite{xu2022epvisa}. From Fig.~\ref{fig:revenuedt}, we can understand the reason for the inefficiency of the PViSA. The PViSA mechanism always selects the AV with the highest valuation in the online submarket to synchronize while ignoring the potential surplus in the offline submarket. 

\subsubsection{Performance Evaluation under Different System Parameters}

By setting the number of tasks to 1, the achieved surplus in the system under different market parameters is demonstrated in Fig.~\ref{fig:sw}. As the market size becomes larger, both the number of buyers in Fig.~\ref{fig:av} and the number of sellers in Fig.~\ref{fig:ar}, the surplus achieved by the proposed framework becomes increasingly higher. A larger market size will lead to more competition among traders in the market, resulting in a higher surplus for RSUs in providing services. As we can observe, the generative AI-empowered simulations can increase surplus by at least 150\% compared with the simulations without generative AI. In addition, the proposed MTEPViSA mechanism can improve more than half of the surplus compared with the PViSA mechanism. Furthermore, the  generative score also has a substantial impact on the surplus in the simulation of the system, as shown in Fig.~\ref{fig:sa}. Therefore, the mechanism can not only select AVs and virtual simulators with a high valuation but also fine-tune the AI models  of AVs for higher accuracy.

\subsubsection{Performance Evaluation under Different System Settings}
In Fig.~\ref{fig:simulationrevenue}, we evaluate the performance of the proposed mechanism under different system settings. As illustrated in Fig.~\ref{fig:revenuetotal}, the total revenue of the virtual simulator finally increases as the number of tasks increases. The proposed MTEPViSA can double the surplus compared with the PViSA. As the number of tasks becomes higher, the performance gap between the proposed METPViSA and the EPViSA becomes larger. From Fig.~\ref{fig:revenuear1}, we can observe that the growth points of the surplus mainly rely on the surplus obtained from provisioning traffic simulation results. From Fig.~\ref{fig:revenuedt}, we can understand the reason for the inefficiency of the PViSA. The PViSA mechanism always selects the AV with the highest valuation in the online submarket to synchronize while ignoring the potential surplus in the offline submarket. Finally, the MTEPViSA mechanism and the PViSA mechanism can achieve a higher surplus in provisioning driving simulations by addressing the asymmetric information in the offline submarket. As illustrated in Fig.~\ref{fig:real_simulationrevenue}, using the synthesized datasets of TSDreambooth, we obtain the total, DT, traffic simulation, driving simulations surpluses for the generated, background modification, and re-contextualization datasets. We can see that the growth trend of the surplus for each mechanism in the figure is similar to that in the simulation. However, since the quality of the data set in the real experiment is not as good as in the simulation, the distribution of the results in the experiment is relatively uneven. This can be seen most clearly in Fig.~\ref{fig:real_revenuetotal} and Fig.~\ref{fig:real_revenuear1}. These results also show that although the generated datasets improve the performance of AI models in AVs compared to the original datasets, the improvement is not homogeneous depending on the datasets generated by different preferences.

\section{Conclusion}\label{sec:con}
In this paper, we have proposed a generative AI-empowered autonomous driving architecture for the vehicular Metaverse. In this architecture, we have proposed the multi-task DT offloading model for reliably executing AVs' DT tasks with different requirements at RSUs. In addition, we have leveraged the generative AI models to synthesize diverse and conditioned driving simulation datasets for AVs' offline training. Finally, we have devised the multi-task enhanced auction-based mechanism to incentivize RSUs to support the simulation systems for autonomous driving systems. The property analysis has validated that the proposed mechanism is strategy-proof and adverse-selection free. The experimental results have demonstrated that the proposed mechanism can increase the social surplus by around 150\%.

\bibliographystyle{IEEEtran}
\bibliography{main}

\end{document}